\documentclass[11pt]{article}
\usepackage[margin=1.2in]{geometry}
\usepackage[colorlinks,urlcolor=blue,linkcolor=blue,citecolor=blue]{hyperref}

\usepackage{color,array}
\usepackage{amsmath,amssymb,amsfonts}
\usepackage{graphicx}
\usepackage{subcaption}
\usepackage{array}

\usepackage{textcomp}
\usepackage{bm}

\usepackage{url}
\usepackage{cite}
\usepackage{float}

\newcommand{\R}{\mathbb{R}}

\usepackage{comment}
\usepackage[ruled,linesnumbered]{algorithm2e}

\usepackage{multicol}
\usepackage{color}
\newcounter{multifig}
\usepackage{thmtools,thm-restate}
\usepackage{multirow}
\usepackage{mathtools}

\usepackage{lipsum}

\usepackage{array}
\usepackage{makecell}
\usepackage[thinc]{esdiff}

\newcommand\norm[1]{\left\lVert #1 \right\rVert}

\usepackage{fancyhdr}
\usepackage{amsthm}
\newtheorem{theorem}{Theorem}[section]

\newtheorem{remark}[theorem]{Remark}
\newtheorem{example}[theorem]{Example}
\newtheorem{assumption}[theorem]{Assumption}

\newtheorem{query}[theorem]{Query}
\newtheorem{lemma}[theorem]{Lemma}

\def\BibTeX{{\rm B\kern-.05em{\sc i\kern-.025em b}\kern-.08em
    T\kern-.1667em\lower.7ex\hbox{E}\kern-.125emX}}

\usepackage{environ}
\usepackage{mdframed}

\usepackage{authblk}

\renewenvironment{abstract}{%
  \begin{mdframed}[linewidth=0.8pt, innerleftmargin=6pt, innerrightmargin=6pt,
                   innertopmargin=6pt, innerbottommargin=6pt]%
  \small\textbf{Abstract---}%
}{%
  \end{mdframed}%
  \vspace{4pt}%
}

\newenvironment{IEEEkeywords}{%
  \vspace{2pt}\noindent\small\textbf{Index Terms---}%
}{%
  \par\vspace{6pt}%
}

\newcommand{\IEEEPARstart}[2]{\textbf{#1}#2}

\begin{document}

\title{\textbf{Collaborative AI Agents and Critics for Fault Detection and Cause Analysis in Network Telemetry}}

\author[1,2]{Syed Eqbal Alam}
\author[1]{Zhan Shu}

\affil[1]{Department of Electrical and Computer Engineering,
          University of Alberta, Edmonton, Alberta, Canada}
\affil[2]{SheQAI Research, Edmonton, Alberta, Canada}

\date{}

\maketitle
\thispagestyle{fancy}

\begin{abstract}
We develop algorithms for collaborative control of AI agents and critics in a multi-actor, multi-critic federated multi-agent system. Each AI agent and critic has access to classical machine learning or generative AI foundation models. 
The AI agents and critics collaborate with a central server to complete multimodal tasks such as fault detection, severity, and cause analysis in a network telemetry system, text-to-image generation, video generation, healthcare diagnostics from medical images and patient records, etcetera. The AI agents complete their tasks and send them to AI critics for evaluation. The critics then send feedback to agents to improve their responses. Collaboratively, they minimize the overall cost to the system with no inter-agent or inter-critic communication. AI agents and critics keep their cost functions or derivatives of cost functions private. Using multi-time scale stochastic approximation techniques, we provide convergence guarantees on the time-average active states of AI agents and critics. The communication overhead is a little on the system, of the order of $\mathcal{O}(m)$, for $m$ modalities and is independent of the number of AI agents and critics. Finally, we present an example of fault detection, severity, and cause analysis in network telemetry and thorough evaluation to check the algorithm's efficacy. 
\end{abstract}


\section{Introduction} 
\IEEEPARstart{F}{oundation models} are generic generative artificial intelligence models trained on a vast corpus of data. They need significant computing infrastructure to train them. Moreover, they can be adapted to perform a wide variety of tasks \cite{Rishi2021}. Foundation models, such as large language and vision models, have shown remarkable capabilities to generate human-like text, code, videos, audio, language understanding, image classification, etcetera.  Some of the examples of foundation models are OpenAI's GPT-4, Google Deepmind's Flamingo \cite{Alayrac2024}, Meta's Llama3.2, Mistral, IBM's Granite \cite{Mishra2024}, to name a few. 

The foundation model's generated and processed data type is called a modality. Based on the modalities, we categorize foundation models into unimodal and multimodal. Unimodal foundation models generate and process one type of data or modality, such as generating text, whereas multimodal models generate and process several modalities, such as text, image, video, etc. Developing uni-modal foundation models is a challenging task that involves resolving data and security issues, fairness, model alignment, etc. Furthermore, sometimes, the model gives fictional outputs with high confidence (referred to as {\em hallucination}). Developing multimodal foundation models that support multiple modalities increases complexity several times more than uni-modal models because of processing different modalities' data. Furthermore, it needs to address other issues of each modality it supports and their integration issues. As a basic example, in the multimodal model for text, vision, and audio,  for pre-processing input data, the text uses words as tokens; vision uses pixels as tokens; audio uses sound frequencies as tokens; creating these tokens involves different techniques. Instead of developing foundation models that work for many modalities, we can work with specialized models that do specific tasks very well and then integrate these foundation models of desired modalities; this can be done through multi-agent collaboration. Each agent can have access to one or more foundation models; we call such agents {\em AI agents}. Recently, systems utilizing foundation model-based single autonomous agents are progressing towards multiple collaborative agents. Every agent of the foundation models-based multi-agent system can have a unique capability, and they collaboratively can provide emerging capabilities for multimodality. Moreover, using multiple collaborative agents increases the accuracy of the tasks' outcome \cite{Junyou2024}.

Most of the LLM-based multi-agent systems work is centered around text generation, and a few are in multimodal settings \cite{Taicheng2024}.
We propose stochastic algorithms to control the number of AI agents and critics in a federated multi-agent system. In the proposed model, multiple agents and critics collaborate with a central server to complete multimodal tasks and aim to minimize the overall cost to the system. The AI agents and critics have access to foundation models of different modalities or classical machine learning models. The central server keeps track of the number of active agents and critics at a time step and broadcasts a {\em feedback signal} for a modality in the system at each time step. 
\begin{figure}
    \centering
\includegraphics[width=1.0\columnwidth, trim=5.2cm 13.8cm 3cm 4.1cm, clip]{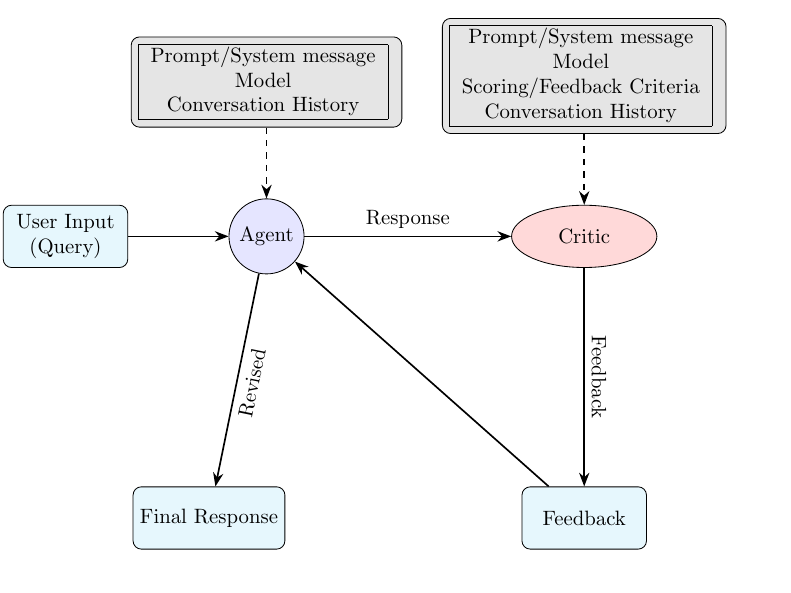}
    \caption{Block diagram for essential components of AI agent-critic interaction. Models represent the foundation models; for example, large language models, large vision models, etc.}
    \label{fig:diagram_AI-agents_critics}
\end{figure}
Users pass their questions or queries to AI agents; an agent generates its response and sends it to an active critic. The critic then evaluates the agent's response and sends feedback to the agent. The agent considers the critic's feedback and improves its response if it agrees with the critic's evaluation. With feedback, critics can also include a score, a numeric value. For example,
after receiving, the agent's response, the critic evaluates it, assigns a score, and sends feedback with score to the agent. The agent considers the critic's score and feedback and improves its response if the score is below a predefined threshold. If the score of the generated response is above the threshold, agents send it as final response for the task. 
The block diagram for essential components of AI agents and critics and their interactions is illustrated in Figure \ref{fig:diagram_AI-agents_critics}. Generative AI-enabled agents and critics have models (LLMs or LVMs), prompts to perform the task, and conversation history in their memory. In addition, the critics can also have evaluation criteria with different score weightage for response accuracy, task relevance, completeness, clarity, etc. 
Agents and critics make their decisions in a probabilistic way whether to participate at a time step or not. The response probabilities of an agent depend on its time-average active states, the (partial) derivatives of its cost function, and the feedback signals from the central server. Analogously, for the critics. The cost functions of an agent are coupled through its time-average active states for all modalities. Furthermore, the cost functions of a critic depend on its time-average active states for the modality it is evaluating. 
Collaboratively, they minimize the overall cost to the system with no inter-agent or inter-critic communication. Agents and critics keep their cost functions or derivatives of cost functions private. Using multi-time scale stochastic approximation techniques, we provide convergence guarantees on the time-average active states of agents and critics. The communication overhead is a little on the system and is independent of the number of agents and critics. 

We present an example of fault detection, severity, and cause analysis in network telemetry system. Network telemetry is an automated process to remotely collect and process network information \cite{Tan2021}.  It can monitor data flow across the network to identify any unusual traffic, anomaly, or fault. Furthermore, it improves networks performance, efficiency, reliability, and security~\cite{Minlan2019}.
Network telemetry is a building block for many network management applications \cite{Ran2020}, such as fault detection and localization \cite{Arzani2018}, congestion control \cite{Yuliang2019}, network health monitoring \cite{Hyun2019}, etc., more applications can be found at \cite{Tan2021, Landau-Feibish2025}.  
Fault detection, severity, and cause analysis is one of the most important problems in networks for safe and reliable operation of the networks. 
We evaluate our model with an open-source network telemetry dataset available at \cite{Putina2021}. We chose dataset $4$ that contains tabular data in a comma-separated values (CSV) format and log files for port state changes. The data has three types of port state changes: Optical transceiver pull and reinsert, admin state change (down and up), and link state changes (plug-in and plug-out of fiber optics cable). The producer devices are named ``leaf'' and ``spine'' with numeric values, such as leaf1, spine2, etcetera. The CSV file has $99$ columns, which include numeric and non-numeric data; to name a few---{\em name, time, bandwidth, bytes-received, bytes-sent, CRC-errors, input-data-rate, input-drops, reliability, state, system-ram-memory, total-cpu-fifteen-minute}, etcetera. Figure \ref{fig:Density-plots} illustrates the density plots of a few columns of all rows of the CSV file of dataset $4$ of \cite{Putina2021}. For better illustration, Figure \ref{fig:columns} presents plots of a few columns for some of the data points.
\begin{figure}[t]
    \centering
    \subfloat[]{\includegraphics[width=0.498\columnwidth]{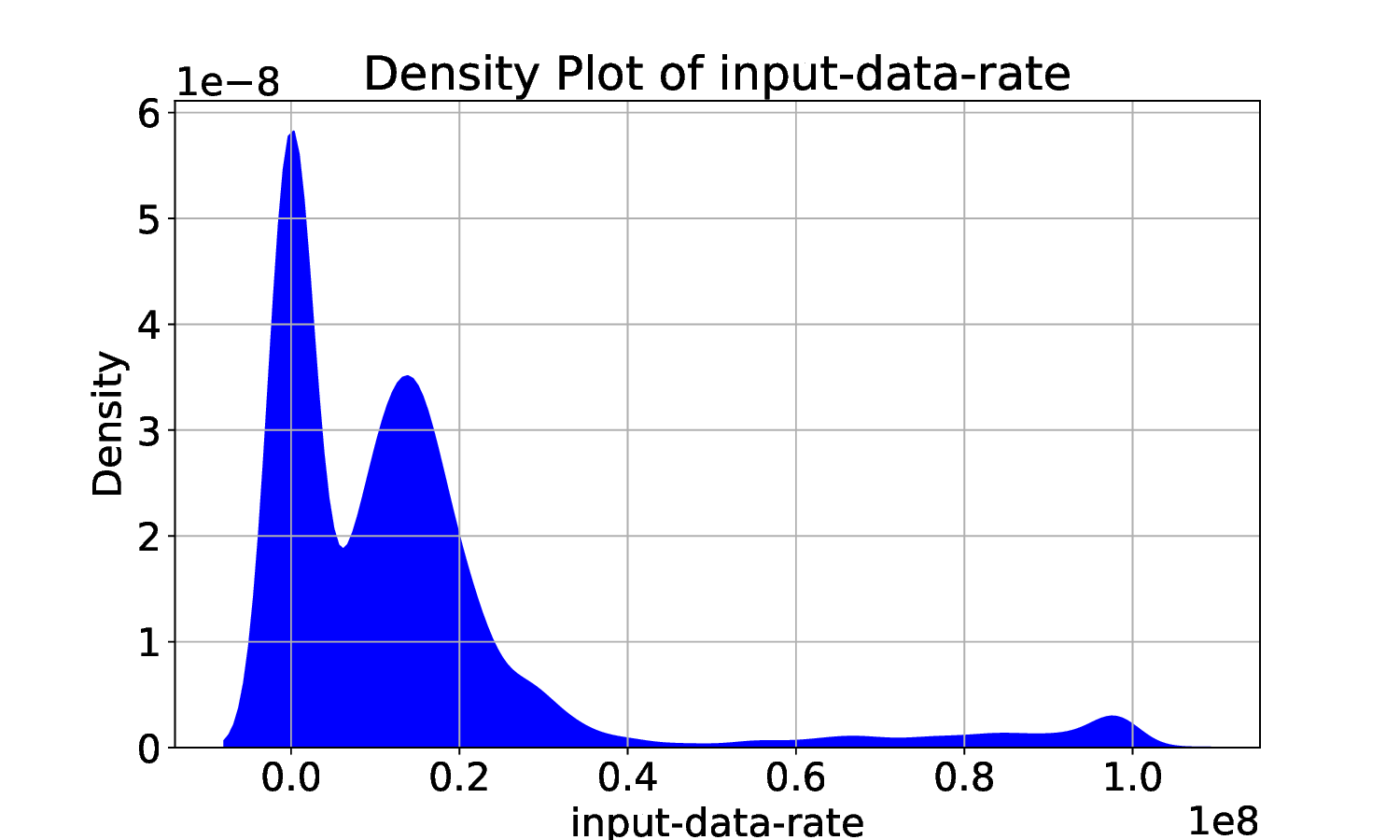}} 
    \subfloat[]{\includegraphics[width=0.498\columnwidth]{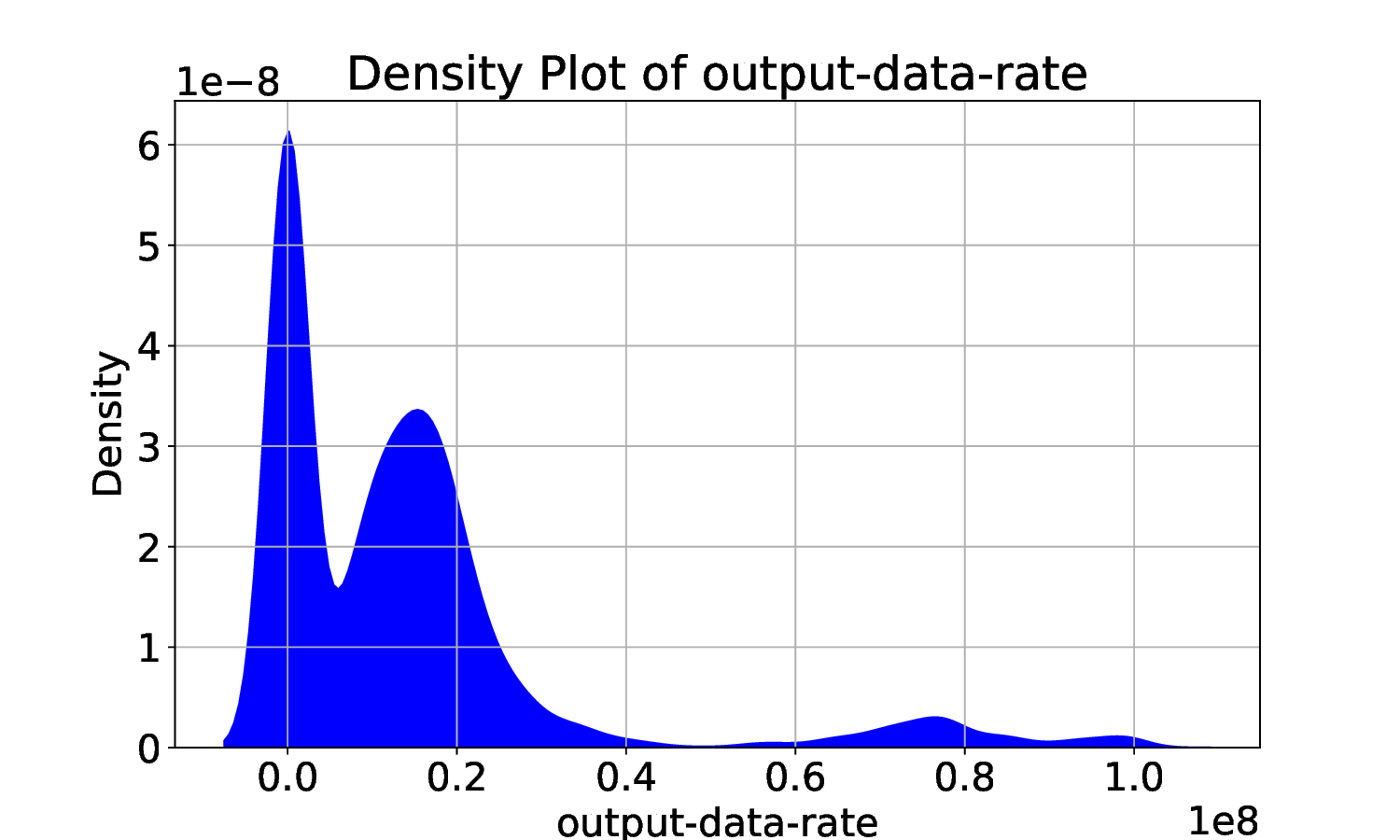}} \\
    \subfloat[]{\includegraphics[width=0.498\columnwidth]{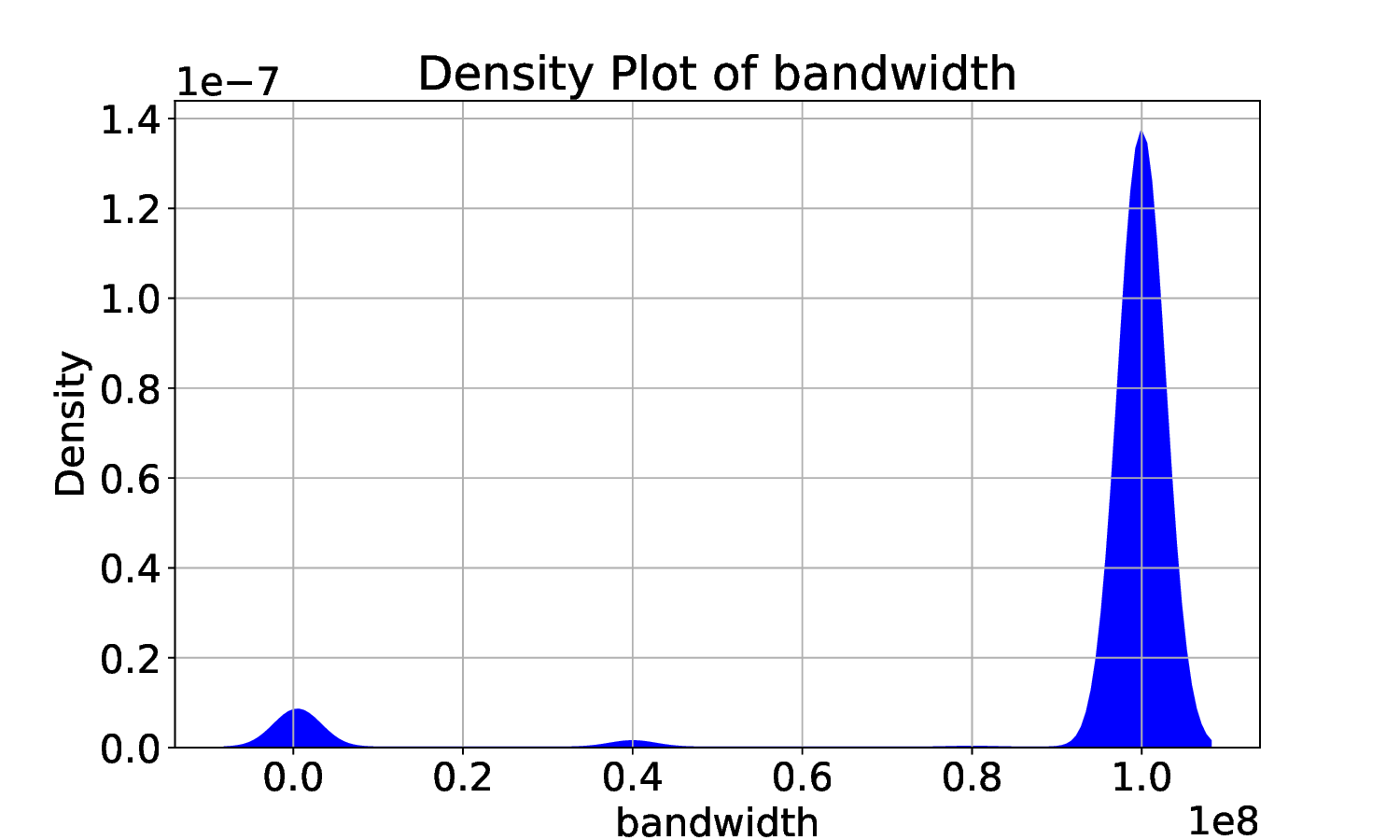}}  
    \subfloat[]{\includegraphics[width=0.498\columnwidth]{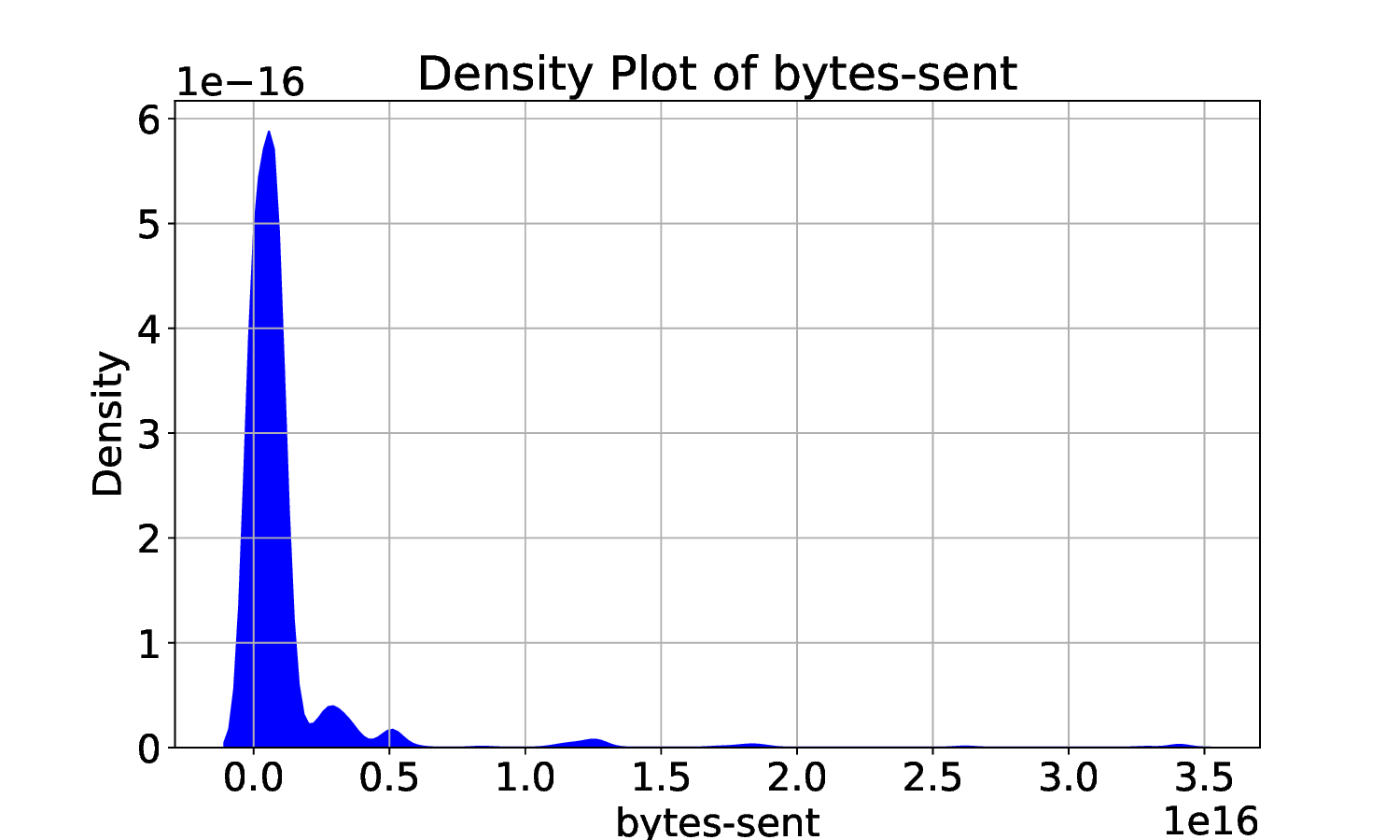}}   
    \caption{Density plots for input data rate, output data rate, bandwidth, and bytes sent for the CSV file (all rows) of the network telemetry dataset $4$ of \cite{Putina2021}.}
    \label{fig:Density-plots}
\end{figure}
\begin{figure}
   \centering \includegraphics[width=0.49\linewidth]{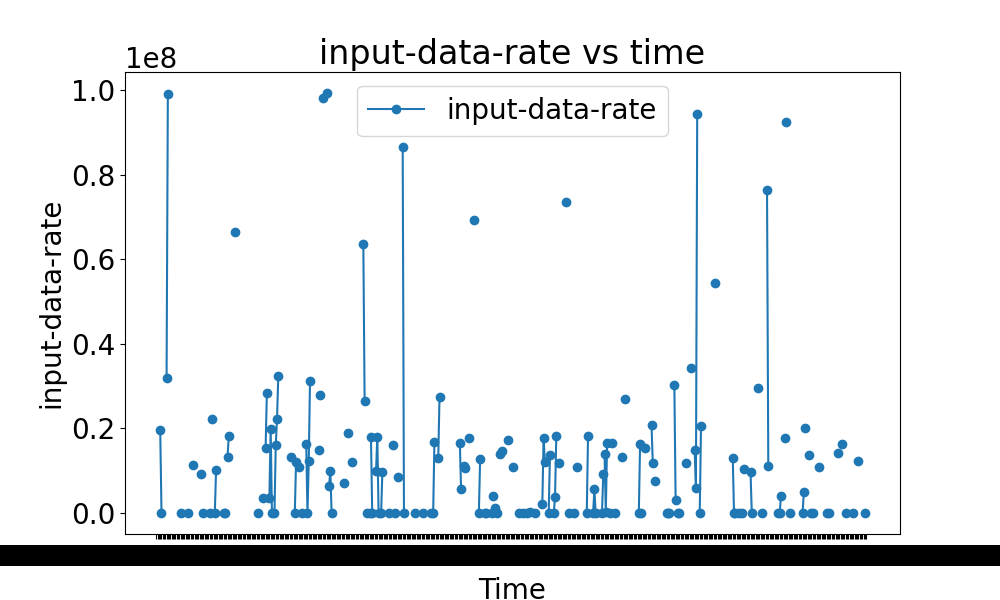}
\includegraphics[width=0.49\linewidth]{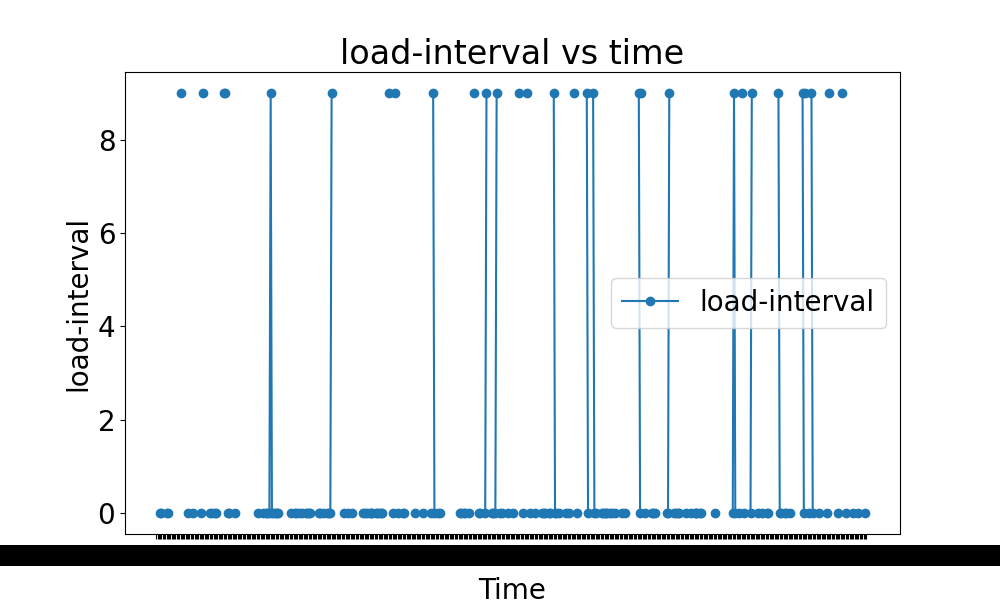}
\includegraphics[width=0.49\linewidth]{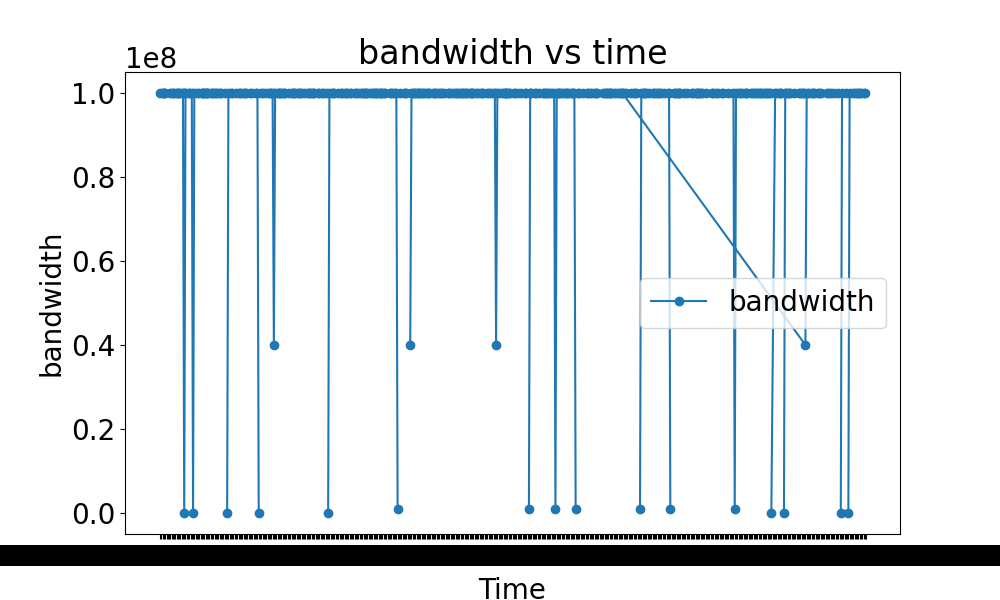}
\includegraphics[width=0.49\linewidth]{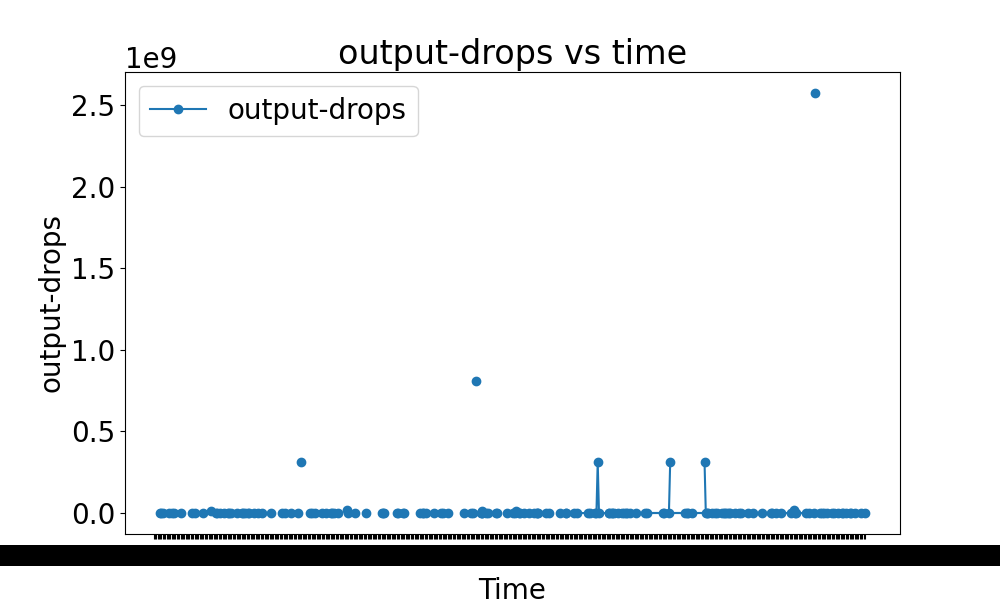}
\caption{The plots for input data rate, load interval, bandwidth, output drops versus time for a few data points from the CSV dataset.}
    \label{fig:columns}
\end{figure}

We evaluate the developed model for fault detection, severity, and cause analysis for the network telemetry dataset in two steps: $(i)$ Fault detection and $(ii)$ Severity and Cause analysis. 
In the fault detection step, we use two approaches: $(a)$ agents enabled with classical machine learning techniques and $(b)$ agents and critics enabled with Generative AI's foundation models. In the Severity and Cause analysis step, we use agents and critics enabled with Generative AI's foundation models. In classical machine learning, the Extreme Gradient Boosting (XG Boosting) model is used; it provides higher accuracy, F1 score, precision, and recall with faster processing than Generative AI's foundation models in the fault detection phase. In classical machine learning based approach, agents are enabled with the Extreme Gradient Boosting model; however, critics are not used (it does not affect the model formulation). Nevertheless, we use both agents and critics in the Generative AI's foundation model-based approaches for fault detection. In the fault detection step, we can use agents enabled with foundation models, and then critics enabled with the XG Boosting or vice versa, but at the cost of higher execution time. One of the advantages of enabling agents with classical machine learning is the speed and accuracy of fault detection, which is crucial in segregating the faults faster. For fault severity and cause analysis, we use agents and critics enabled by foundation models. The developed model is suitable for real-time fault detection, severity, and cause analysis. We make the following principal contributions to the paper.

\subsection*{Contributions:}
\begin{itemize}
\item We developed stochastic iterative algorithms for collaborative control of multiple AI agents and critics in a federated multi-agent system; the agents and critics leverage classical machine learning techniques and generative AI's foundation models to complete multimodal tasks.
\item We provide the convergence guarantees of the AI agents' and critics' time-average active states. Our solution is based on stochastic approximation techniques; using multiple time scales and generalized decreasing step sizes. The developed model incurs a little communication overhead on the system, of the order of $\mathcal{O}(m)$, for $m$ task modalities and it is independent of the number of AI agents and critics. 
\item We present an example of fault detection, severity, and cause analysis on a real-world network telemetry dataset \cite{Putina2021}. The dataset consists of tabular data (CSV files) and log files for port state changes---Optical transceiver (pull and reinsert), admin state change (down and up), and link state changes (plug-in and plug-out of fiber optics cable). We present extensive experimental results to demonstrate the algorithm's efficacy. Nonetheless, we compare the outcomes with Meta's Llama3.2, Mistral, DeepSeek-R1, Llava:7b, IBM's Granite3.2, and Microsoft's Phi4:14b; these are open-source large language models supported by Ollama and Langchain.
\end{itemize}

\section{Notations and Problem description} \label{prob_form}
Let $\mathcal{A} = \{A_1, A_2, \ldots,A_N\}$ be a set of $N$ AI agents in a multi-agent system. Let every agent have access to foundation models, perception module (data input), and have a memory to store their responses/conversation history. These agents collaborate to complete tasks of several modalities, such as audio, text, video, sensor data, etc. Let the total number of modalities be $m$. Furthermore, let us assume that a central server tracks the number of active agents completing the tasks. We call such a setting as {\em federated multi-agent system}. The central server knows the desired number of agents required to complete the tasks of each modality; we call it the {\em capacity of the modality}. Let $\mathcal{N}_1$, $\mathcal{N}_2, \ldots,\mathcal{N}_m$ denote the capacities of modalities $1,2,\ldots,m$, respectively. 
Furthermore, we consider a few critics for each modality. The critics of the modality evaluate the quality of agents' responses and provide feedback (with or without a score) to each response. Let the number of critics required to evaluate a task of modality $j$ be denoted by $C_j$, for $j=1,2,\ldots,m$.  Analogous to agents, each critic has access to a foundation model, perception module (data input), and has a memory to store their responses, queries, and agents' responses. The critics can also follow an evaluation criteria for agent's responses for accuracy, relevance to task, completeness and clarity. Let all critics are not required to be active at a time step, and let $\mathcal{M}_1$, $\mathcal{M}_2, \ldots,\mathcal{M}_m$ denote the critics capacities of modalities $1,2,\ldots,m$, respectively. Note that agents do not know capacities of modalities, $\mathcal{N}_1, \ldots,\mathcal{N}_m$; similarly, critics do not know capacities $\mathcal{M}_1, \ldots,\mathcal{M}_m$. Furthermore, the agents need not know the number of participating or active agents at a time step.

Let $\mathbb{N}$ denote the set of natural numbers and $\mathbb{R}$ denote the set of real numbers, and let $\mathbb{R}_+$ denote the set of positive real numbers. Let $k \in \mathbb{N}$ denote the discrete time steps. For $i=1,2,\ldots,N$ and $j=1,2,\ldots,m$, let the output response of agent $A_i$ for the task of modality $j$ at time step $k$ be denoted by $R_{ji}(k)$. Let the feedback provided by the critic $\nu$ for agent $A_i$'s response $R_{ji}(k)$ be denoted by $F_{j\nu}^i(k)$, $\nu=1,2,\ldots,C_j$, $j=1,2,\ldots,m$, $i=1,2,\ldots,N$. 
Let $X_{ji}(k) \in \{0,1\}$ denote the active states of agent $A_i$ at time step $k$ for modality $j$, let $X_{ji}(k) = 1$ denote the agent $A_i$ is active to complete the task of modality $j$ at time step $k$, whereas $X_{ji}(k) = 0$ denotes the agent is not active at time step $k$ to complete the task. Let $x_{ji}(k) \in [0, 1]$ denote the time-average number of times $A_i$ was active to complete the task of modality $j$ until time step $k$.
Let $Y_{j\nu}(k) \in \{0,1\}$ denote the active states of critic $\nu$ evaluating modality ${j}$ at time step $k$, and let $Y_{j\nu}(k) = 1$ denote critic $\nu$ evaluating agent's response of modality ${j}$ is active at time step $k$, whereas $Y_{j\nu}(k) = 0$ denote the critic $\nu$ is inactive or not participating for evaluation of agent's response of modality ${j}$ at time step $k$. Let $y_{j\nu}(k) \in [0, 1]$ denote the time-average number of times critic $\nu$ was active to evaluate the task of modality $j$ until time step $k$.

For agents $i = 1, 2, \ldots, N$, and modality $j = 1, 2, \ldots, m$, we define ${x}_{ji}(k)$ as follows:
\begin{align} \label{eq:average_alloc}
{x}_{ji}(k) := \frac{1}{k+1} \sum_{\ell=0}^k X_{ji}(\ell).
\end{align}

Analogously, for critics $\nu = 1, 2, \ldots, \mathcal{M}_j$ and modality $j = 1, 2, \ldots, m$, we define:
\begin{align} \label{eq:average_alloc_Critic}
{y}_{j\nu}(k) := \frac{1}{k+1} \sum_{\ell=0}^k Y_{j\nu}(\ell).
\end{align}

For $x_{ji} \in [0,1]$, let $\mathbf{x}_j = (x_{j1}, x_{j2}, \ldots, x_{jN})$ and $\mathbf{x} =(\mathbf{x}_1, \mathbf{x}_2, \ldots, \mathbf{x}_m)$.
Analogously, for $y_{j\nu} \in [0,1]$, let $\mathbf{y}_j = (y_{j1}, y_{j2}, \ldots, y_{jC_j})$ and $\mathbf{y} =(\mathbf{y}_1, \mathbf{y}_2, \ldots, \mathbf{y}_m)$.
Let agent $A_i$ has cost function $f_i: [0,1]^m \to \mathbb R_+$, $i=1,2,\ldots,N$; it captures the cost using classical machine learning models or the foundation models and other computational costs of agent $A_i$ performing the tasks of $m$ modalities. Additionally, let critic $\nu$ of modality $j$ has cost function $g_{j\nu}: [0,1] \to \mathbb R_+$,  $\nu = 1,2,\ldots,C_j$, $j=1,2,\ldots,m$; it captures the cost of the critic of modality $j$. For the cost functions, we have the following assumption to lead to unique optimal solution:
\begin{assumption}
The cost functions $f_i$ and $g_{j\nu}$ are twice continuously differentiable, strictly convex, and increasing in each variable, for $i=1,2,\ldots,N$, $\nu = 1,2,\ldots,C_j$, and $j=1,2,\ldots,m$.
\end{assumption}

We develop an iterative stochastic algorithm for the multi-actor multi-critic model wherein agents keep their cost functions and derivatives of the cost functions private even with critics and the central server; analogously, the critics also keep their cost functions and derivatives of the cost functions private. Nonetheless, they share their willingness to participate at each time step.

We formulate the problem as an optimization problem as follows:
\begin{align}
\begin{split} \label{prob:objective_fn}
\min_{\bf{x}, \bf{y}} \quad &\sum_{i=1}^{N} f_i(x_{1i},\ldots,x_{mi}) + \sum_{j=1}^{m} \sum_{\nu=1}^{C_j} g_{j\nu}(y_{j\nu}),
\\ \mbox{subject to } \quad &\sum_{i=1}^{N} x_{1i}  =  \mathcal{N}_1\\ &\vdots\\ &\sum_{i=1}^{N} x_{mi}  =  \mathcal{N}_m,\\
 &\sum_{\nu=1}^{C_1} y_{1\nu}  =  \mathcal{M}_1\\ &\vdots\\ &\sum_{\nu=1}^{C_m} y_{m\nu}  =  \mathcal{M}_m,
\\  &x_{ji} \in [0, 1], \quad i=1,\ldots,N \mbox{ and } j=1,\ldots,m,
\\  &y_{j\nu} \in [0, 1], \quad \nu=1,\ldots,C_j \mbox{ and } j=1,\ldots,m.
\end{split}
\end{align}

As the constraint set of \eqref{prob:objective_fn} is compact, therefore, optimal solutions exist. The assumption that the cost functions $f_i$ and $g_{j\nu}$ are strictly convex and increasing leads that $\sum_{i=1}^{N} f_i(x_{1i},\ldots,x_{mi})$ + $\sum_{j=1}^{m} \sum_{\nu=1}^{C_j} g_{j\nu}(y_{j\nu})$ is also strictly convex and increasing in decision variables $x_{ji}$ and $y_{j \nu}$. Therefore, the solution to \eqref{prob:objective_fn} is a unique optimal solution.
For $i=1,2,\ldots,N$, $\nu = 1,2, \ldots, C_j$, and $j=1,2\ldots,m$, let $x_{ji}^{*} \in (0,1]$ and ${y}_{j\nu}^{*} \in (0,1]$ denote the solution to \eqref{prob:objective_fn}. 
The algorithm determines the active states of agents $X_{ji}(k)$ and critics $Y_{j\nu}(k)$, and it ensures that the time-average number of active agents and active critics (defined in \eqref{eq:average_alloc} and \eqref{eq:average_alloc_Critic}) converge to their respective optimal values and achieve an overall minimum cost to the system, stated as follows:
\begin{align*}
&\lim_{k\to \infty} {x}_{ji}(k) = {x}_{ji}^{*}, 
 \lim_{k\to \infty} {y}_{j \nu}(k) = {y}_{j \nu}^{*}, \\& \mbox{for } i = 1, 2, \ldots, N, \nu = 1, 2, \ldots, C_j, j=1,2,\ldots,m,
\end{align*} $\mbox{ and }$
\begin{align*}
&\lim_{k\to \infty}  \sum_{i=1}^{N} f_i(x_{1i}(k),\ldots,x_{mi}(k)) + \lim_{k\to \infty}  \sum_{j=1}^{m} \sum_{\nu=1}^{C_j} g_{j\nu}(y_{j\nu}(k)) \\&=  \sum_{i=1}^{N} f_i(x_{1i}^*,\ldots,x_{mi}^*) + \sum_{j=1}^{m} \sum_{\nu=1}^{C_j} g_{j\nu}(y_{j\nu}^*). 
\end{align*}

We first find the Lagrangian of the optimization problem \eqref{prob:objective_fn} and then follow a similar analysis as in \cite{Syed2018_B} and find out that the (partial) derivatives of the cost functions of all agents collaborating to complete tasks of the same modality reach consensus at its optimal value; the same analysis holds for critics. We obtain the following results, for $i,u \in \{1, 2, \ldots, N\}$, $\nu, \ell \in \{1, 2, \ldots, C_j\}$, and $j=1,2,\ldots,m$:
\begin{align}  \label{optimality_cond_b1}
	\frac{\partial}{\partial x_{ji}} f_i(x_{1i},\ldots,x_{mi}) \Big \vert_{x_{ji}= x_{ji}^*}  = \frac{\partial}{\partial x_{ju}} f_u(x_{1u},\ldots,x_{mu}) \Big \vert_{x_{ju} = x_{ju}^*}.
\end{align} 
 \mbox{ and }
 \begin{align}  \label{optimality_cond_critic}
	\frac{d}{ d y_{j\nu}}g_{j\nu}(y_{j\nu}) \Big \vert_{y_{j\nu}= y_{j\nu}^*} = \frac{d}{d y_{j\ell}}g_{j\ell}(y_{j\ell}) \Big \vert_{y_{j\ell}= y_{j\ell}^*}.
\end{align} 
Karush-Kuhn-Tucker (KKT) conditions are satisfied by the consensus of derivatives (refer \eqref{optimality_cond_b1} and \eqref{optimality_cond_critic}) of the cost functions that are necessary and sufficient conditions of optimality of the optimization Problem \eqref{prob:objective_fn}; a similar analysis is done in \cite{Syed2018_B, Wirth2014, Syed2018, Alam2023_CDC}, further details of KKT conditions can be found at Chapter 5.5.3 \cite{Boyd2004}.

\section{Algorithm Design}  \label{bin_imp}
This section presents the developed algorithms for AI agents, AI critics, and the central server. AI agents and critics are Generative AI's foundation model or classical machine learning models enabled. AI agents have access to models of different modalities and complete multimodal tasks. Moreover, we consider that there are $m$ sets of critics, and critics of each set evaluate the agents' responses for one particular modality. Each AI agent runs its algorithm to calculate its probabilistic intent to participate in completing its (multimodal) tasks. Analogously, the critics run their algorithms to calculate their probabilistic intent to evaluate the responses of the AI agents and provide feedback on each response. After receiving feedback from the critics, agents update their responses if they agree with the evaluation. Notice that the algorithm can also be updated to include score with feedback. If the score is below a threshold, then the agent updates its response; otherwise, it sends the response as final. Moreover, the threshold could be set by the algorithm designer based on the prediction accuracy of the models. 

The central server keeps track of the number of active AI agents and critics at each time step. The feedback signal of $j$'th modality depends on its value at the previous time step, a step size, modality capacity, and the total number of active participants (agents or critics). The central server updates and broadcasts feedback signals for each modality to agents and critics at every time step. 
Figure \ref{fig:diagram_cS} presents the model as a block diagram for agents, critics, and the central server.
\begin{figure}
    \centering    \includegraphics[width=1.0\columnwidth, trim=5.2cm 7cm 3cm 4.3cm, clip]{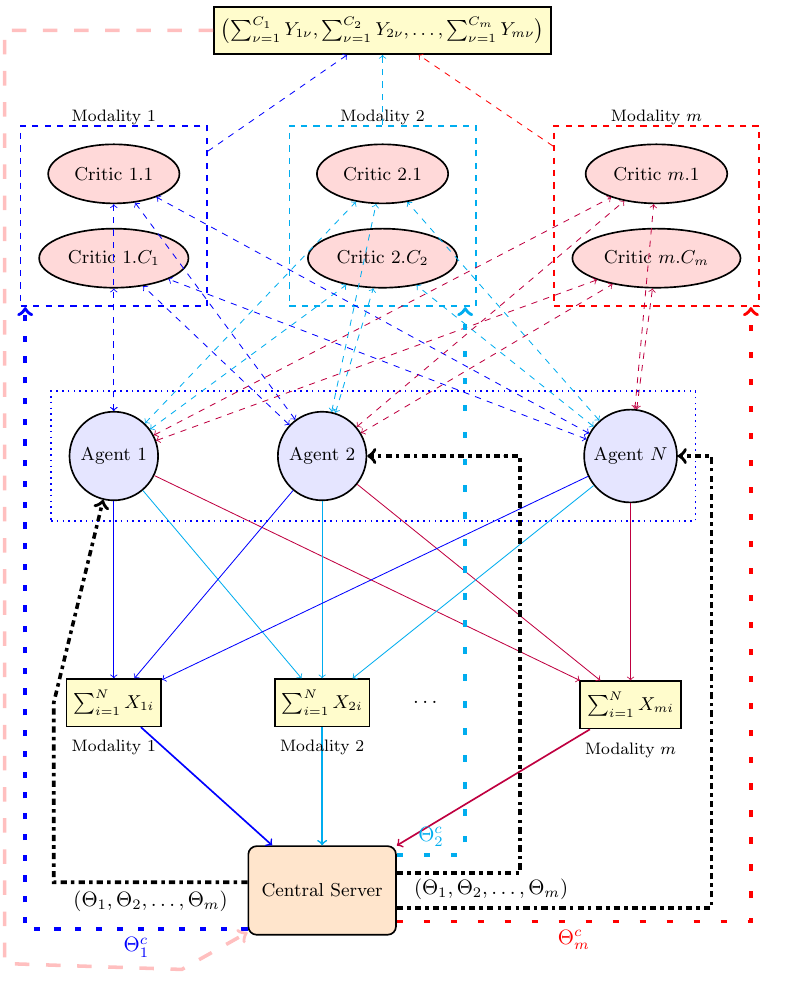}
    \caption{Federated multi-agent system's block diagram: AI agents and critics coordinate with a central server to complete multimodal tasks.}
    \label{fig:diagram_cS}
\end{figure}

We now define a few additional notations here to describe the algorithms. 
Let $\Theta_j(k)$ denote {\em feedback signal for AI agents} completing modality $j$'s task, and let $\Theta^c_j(k)$ denote {\em feedback signal for AI critics} of modality $j$ at time step $k$. Let $\eta_{j}(k)>0$ denote the step size used to update feedback signal $\Theta_j(k)$ for agents and $\eta^c_j(k)>0$ denote step size used to update $\Theta^c_j(k)$ for critics at time step $k$. Let  $\eta_{0_j}>0$ and $\eta^c_{0_j}>0$ be the initial values for step sizes of feedback signals for agents and critics, respectively. Let $\alpha_j>0$ denote a step size decay coefficient, and let $\beta_j \in (1/2,1)$ denote a decay exponent for step sizes of feedback signals for agents. We define the step size for agents as $\eta_j(k) := \frac{\eta_{0_j}}{(\alpha_j k +1)^{\beta_j}}$. Analogously, we define the step size of feedback signals for critics of modality $j$, let $\alpha^c_j>0$ denote the critic's step size decay coefficient, and let $\beta^c_j \in (1/2,1)$ denote the decay exponent of critics. We define the step size for critics as $\eta^c_j(k) := \frac{\eta^c_{0_j}}{(\alpha^c_j k +1)^{\beta^c_j}}$. The feedback signal is updated for agents and critics, as described in \eqref{omega_b1} and \eqref{omega_b2},  respectively.  The central server broadcast them at each time step. We present the algorithmic details for AI agents and critics in the following subsections.

\subsection{Algorithm of AI agents}
Let  $\sigma_{ji}(\Theta_j(k),x_{1i}(k),\ldots,x_{mi}(k))$ denote the probability distribution of agent $A_i$ at time step $k$. After receiving the feedback signal from the central server, agent $A_i$ calculates $\sigma_{ji}(\Theta_j(k),x_{1i}(k),\ldots,x_{mi}(k))$ to decide whether to be active or not for completing tasks of  modality $j$ at time step $k+1$; $\sigma_{ji}$ depends on feedback signals of modality $j$ for agents, agents $A_i$'s time-average number of active states $(x_{1i}(k),\ldots,x_{mi}(k))$ and the partial derivative of the cost function $f_i$ for $x_{ji}$. Agent $A_i$ calculates its response probability for all modalities $j$ at each time step $k \in \mathbb{N}$, for $i=1,2,\ldots,N$, see \eqref{eq:prob}. Moreover, 
agent $A_i$ finds out the outcome of the Bernoulli trial for modality $j$ at time step $k$, outcome $1$ occurs with probability $\sigma_{ji}(\Theta_j(k),x_{1i}(k),\ldots,x_{mi}(k))$ and outcome $0$ occurs with probability $1-\sigma_{ji}(\Theta_j(k),x_{1i}(k),\ldots,x_{mi}(k))$. Outcome value $1$ represents that the agent enters into an active state at time step $k+1$, and $0$ represents the agent will stay inactive, described as follows:
 \begin{align} \label{eq:alloc-update}
&X_{ji}(k+1) \nonumber \\&=
\begin{cases} 
1 \quad \text{with probability }  \sigma_{ji}(\Theta_j(k),x_{1i}(k),\ldots,x_{mi}(k));\\ 
0 \quad \text{with probability }  1 - \sigma_{ji}(\Theta_j(k),x_{1i}(k),\ldots,x_{mi}(k)).
\end{cases}
\end{align}
Analogous to modality $j$, agents follow this process for every modality at each time step. The algorithm for agents is presented in Algorithm \ref{algo3}. 

\begin{algorithm} \SetAlgoLined Input:
		Capacities $\mathcal{N}_{1}, \ldots, \mathcal{N}_{m}$, $\mathcal{M}_{1}, \ldots, \mathcal{M}_{m}$; \ States $ X_{11}(k), \ldots, X_{mN}(k)$, $ Y_{11}(k), \ldots, Y_{j\nu}(k)$, for $k \in \mathbb{N}$, $\nu =1,2,\ldots,C_j$, $j =1,2,\dots,m$.
		
		Output:
		Feedback signals $\Theta_{1}(k+1), \Theta_{2}(k+1), \ldots, \Theta_{m}(k+1)$, $\Theta^c_{1}(k+1), \Theta^c_{2}(k+1), \ldots, \Theta^c_{m}(k+1)$, for $k \in \mathbb{N}$.
		
		Initialization: Initialize feedback signals $\Theta_{1}(0)$, $\Theta_{2}(0), \ldots, \Theta_{m}(0)$ and $\Theta^c_{1}(0)$, $\Theta^c_{2}(0), \ldots, \Theta^c_{m}(0)$ with a real number.
		
		\ForEach{$k \in \mathbb{N}$}{
			
			\ForEach{modality $j = 1, 2, \ldots, m$}{
				\begin{itemize}	
				\item[--] update $\Theta_j(k+1)$ according to \eqref{omega_b1} \\and broadcast it in the network;
            \item[--] update $\Theta^c_j(k+1)$ according to \eqref{omega_b2} \\and broadcast it in the network.
            \end{itemize}        
		} }
		\caption{Algorithm of the central server.}
		\label{algoCU2}
	\end{algorithm}
    
\begin{algorithm} \SetAlgoLined Input:
		$\Theta_{1}(k), \Theta_{2}(k), \ldots, \Theta_{m}(k)$;
        Critic's feedback $F_{11}^1(k),\ldots, F_{1C_1}^N(k),\ldots, F_{mC_m}^N(k)$  
		
		Output: States $X_{1i}(k+1), \ldots, X_{mi}(k+1)$ and agents responses $R_{1i}(k), \ldots, R_{mi}(k)$,  for $k \in \mathbb{N}$.
		
		Initialization: $count_{ji} \leftarrow 0$, $X_{ji}(0) \leftarrow 1$ and
		${x}_{ji}(0) \leftarrow X_{ji}(0)$, $i = 1, \ldots, N$, critic's feedback $F_{j\nu}^i(0) \leftarrow Null$, for
		$j = 1, \ldots, m$, and critic $\nu$.
		
		\ForEach{$k \in \mathbb{N} $}{
			
		\ForEach{modality $j = 1, 2, \ldots, m$}{
	Calculate $\sigma_{ji}(\Theta_j(k),x_{1i}(k), \ldots, x_{mi}(k))$ as in \eqref{eq:prob}. Generate Bernoulli independent random variable $B_{ji}(k)$ as follows:
     
$B_{ji}(k) \leftarrow
\begin{cases} 
1 \quad \text{w. p. }  \sigma_{ji}(\Theta_j(k),x_{1i}(k), \ldots, x_{mi}(k));\\ 
0 \quad \text{w. p. }  1 - \sigma_{ji}(\Theta_j(k),x_{1i}(k), \ldots, x_{mi}(k)).
\end{cases}
$
				
	\eIf{ $B_{ji}(k) = 1$}{Agent $A_i$ enters into active state, update: $X_{ji}(k+1) \leftarrow 1$\; 
    \While{$count_{ji}<2$}{
            \If{$count_{ji}=0$ /*initial response*/}{Work on task of modality $j$, store your response in $R_{ji}(k)$ and send it to critics for evaluation\;
            $count_{ji} \leftarrow 1$\;}
   \Else{
     Consider critic's feedback $F_{j\nu}^i(k)$ to improve your response $R_{ji}(k)$\;
     Revise the response if you agree with critic's evaluation\;
     Send the response to the user or save it to the response file\;
       }
       Receive an active critic $\nu$'s feedback $F_{j\nu}^i(k)$ for response $R_{ji}(k)$\;
              
     }}{$X_{ji}(k+1) \leftarrow 0$;}
    $count_{ji} \leftarrow 0$\;
    } 
}
\caption{Algorithm of multimodal AI agent $i$.}
		\label{algo3}
	\end{algorithm}
The central server updates the feedback signal $\Theta_j(k + 1)$ for agents of modality $j$ as follows:
	\begin{align} \label{omega_b1}
			&\Theta_j(k+1) \triangleq \Theta_j(k) -  \frac{\eta_{0_j}}{(\alpha_j k+1)^{\beta_j}} \left (\sum_{i=1}^N X_{ji}(k+1) - \mathcal{N}_j \right ),
	\end{align}
 $j=1,2,\ldots,m$ and $k \in \mathbb{N}$.

	 After receiving the feedback signal $\Theta_j(k)$ for modality $j$ from the central server at time step $k$, agent $A_i$ calculates probability $\sigma_{ji}(\Theta_j(k),x_{1i}(k),\ldots,x_{mi}(k))$ to make a decision to be active or not to participate in completing the task of modality $j$ at the next time step, the probability distribution of agent $A_i$ for modality $j$ is calculated as follows:
	\begin{align} \label{eq:prob}
		& \sigma_{ji}(\Theta_j(k),x_{1i}(k),\ldots,x_{mi}(k)) \nonumber \\& \triangleq  \Theta_j(k) \frac{ x_{ji}(k)}{ \frac{\partial}{\partial x_{ji}}f_i(x_{1i}(k),\ldots,x_{mi}(k)) \Big \vert_{x_{ji} = x_{ji}(k)}}.
	\end{align}
	Notice that $\Theta_j(k)$ is used to ensure that the probability $\sigma_{ji}(\Theta_j(k),x_{1i}(k),\ldots,x_{mi}(k))$ is in valid probability range, and so $\eta_{0_j}$, $\alpha_j$ and $\beta_j$ in $\Theta_j(k)$ are chosen accordingly. 

\subsection{Algorithm of AI critics}
We now present the algorithm of the critics. Analogous to the AI agents, the AI critics are also generative AI's foundation model enabled. They receive responses from AI agents at each time step and calculate their probabilistic intent to be active or not to evaluate the agents' responses at the next time step. 
AI critics evaluate AI agents' responses and send feedback to them. With a minor update in the algorithm, the critics can also send scores with their feedback. They can provide scores based on criteria---such as accuracy, relevance, completeness, clarity, etc., of the responses. It is set by the algorithm designers.

 \begin{algorithm} \SetAlgoLined Input:
		AI agent responses: $R_{11}(k), R_{21}(k) , \ldots, R_{mN}(k)$ and 
        Feedback signals from central server for critics: $\Theta^c_1(k), \ldots, \Theta^c_m(k)$\;
		
		Output: Critic's feedback $F_{j\nu}^i(k)$, for $i=1,2,\ldots,N$, $\nu=1,2\ldots,C_j$, $j=1,2,\ldots,m$, and $k \in \mathbb{N}$.
		
		Initialization:       Critics' feedback $F_{j\nu}^i(0) \leftarrow Null $, for $i=1,2,\ldots,N$, $\nu=1,2\ldots,C_j$, $j=1,2,\ldots,m$\;
		
		\ForEach{$k \in \mathbb{N}$}{
			\ForEach{modality $j = 1, 2, \ldots, m$}{
			\ForEach{critic $\nu = 1, 2, \ldots, C_j$}
			{
            Calculate $\sigma^c_{j\nu}(\Theta^c_j(k),y_{j\nu}(k)) \leftarrow \Theta^c_j(k)
				\frac{{y}_{j\nu}(k)}{ \frac{d}{d y_{j\nu}}
					g_{j \nu}(y_{j\nu}(k)) \Big \vert_{y_{j\nu} = y_{j\nu}(k)}} $; 
				
				generate Bernoulli independent random variable
				$B^c_{j\nu}(k)$ with the parameter $\sigma^c_{j\nu}(\Theta^c_j(k),y_{j\nu}(k))$;
     \begin{align*}
B^c_{j\nu}(k) \leftarrow
\begin{cases} 
1 \quad \text{w. p. }  \sigma^c_{j\nu}(\Theta^c_j(k),y_{j\nu}(k));\\ 
0 \quad \text{w. p. }  1 - \sigma^c_{j\nu}(\Theta^c_j(k),y_{j\nu}(k)).
\end{cases}
\end{align*}
				
	\eIf{ $B^c_{j\nu}(k) = 1$}{Critic $\nu$ enters into active state, $Y_{j\nu}(k+1) \leftarrow 1$\;
  Evaluate agent $A_i$'s response for modality $j$, $R_{ji}(k)$. 
Provide feedback $F_{j\nu}^i(k) \leftarrow [feedback]$ for the response $R_{ji}(k)$ to agent $A_i$ based on the set criteria\;
Send the feedback to agent $A_i$\;}{$Y_{j\nu}(k+1) \leftarrow 0$}
        }
    }
}
\caption{Algorithm of AI critic $\nu$} \label{algo_Critic}
	\end{algorithm}

Let $\sigma^c_{j\nu}(\Theta^c_j(k),y_{j\nu}(k))$ denote the probability distribution of critic $\nu$ of modality $j$ at time step $k$. More precisely, critic $\nu$ calculates its probability $\sigma^c_{j\nu}(\Theta^c_j(k),y_{j\nu}(k))$ to be active or not using its time-average number of active states $y_{j\nu}(k)$ of modality $j$ and the derivative of its cost function $g_{j\nu}$ with respect to $y_{j\nu}$; the critic $\nu$ does it for all modalities, $\sigma^c_{j\nu}(\Theta^c_j(k),y_{j\nu}(k))$ is calculated as in \eqref{eq:prob2}, for $k \in \mathbb{N}$. The critic $\nu$ finds out the outcome of Bernoulli trial for modality $j$ at time step $k$---outcome $1$ occurs with probability $\sigma^c_{j\nu}(\Theta^c_j(k),y_{j\nu}(k))$ and outcome $0$ occurs with probability $1-\sigma^c_{j\nu}(\Theta^c_j(k),y_{j\nu}(k))$; based on the value $0$ or $1$, the critic $\nu$  decides whether to participate in evaluating an agent's responses or not. If the outcome is $1$, then the critic shows its willingness to evaluate responses and enters into an active state; otherwise, it stays inactive:
 \begin{align} \label{eq:alloc-update2}
&Y_{j\nu}(k+1) =
\begin{cases} 
1 \quad \text{with probability }  \sigma^c_{j\nu}(\Theta^c_j(k),y_{j\nu}(k));\\ 
0 \quad \text{with probability }  1 - \sigma^c_{j\nu}(\Theta^c_j(k),y_{j\nu}(k)).
\end{cases}
\end{align}

To keep track of the number of active AI critics for a modality, the central server calculates the feedback signal and broadcasts it in the network. The central server updates the feedback signal $\Theta^c_j(k + 1)$ for AI critics evaluating modality $j$ tasks as follows:
	\begin{align} \label{omega_b2}
&\Theta^c_j(k+1) \triangleq \Theta^c_j(k) -  \frac{\eta^c_{0_j}}{(\alpha^c_j k +1)^{\beta^c_j}} \left (\sum_{\nu=1}^{C_j} Y_{j\nu}(k+1) - \mathcal{M}_j \right ),
	\end{align}
 for $j=1,2,\ldots,m$ and $k \in \mathbb{N}$.
	 After receiving the feedback signal $\Theta^c_j(k)$ from the central server at time step $k$, AI critic $\nu$ calculates its probability $\sigma^c_{j\nu}(\Theta^c_j(k),y_{j\nu}(k))$ to make a decision to be active or not to evaluate the agents' responses for modality $j$ at the next time step. Critic $\nu$ calculates the probability distribution for task of modality $j$ as follows:
	\begin{align} \label{eq:prob2}
		& \sigma^c_{j\nu}(\Theta^c_j(k),y_{j\nu}(k)) \triangleq  \Theta^c_j(k) \frac{ y_{j\nu}(k)}{ \frac{d}{d y_{j\nu}}g_{j \nu}(y_{j\nu}(k)) \Big \vert_{y_{j\nu} = y_{j\nu}(k)}}.
	\end{align}
	Similar to feedback signals for AI agents, $\Theta^c_j(k)$ is also used to ensure $\sigma^c_{j\nu}(\Theta^c_j(k),y_{j\nu}(k)) \in (0,1)$.
The critics follow this process for tasks of all modalities at each time step. Algorithm \ref{algo_Critic} presents the algorithm for AI critics.
Finally, following the processes, the time-average active states of AI agents and critics converge to their respective optimal values and a minimum overall cost to the network is obtained with a little communication overhead that is independent of the number of agents and critics in the network. 

\section{Convergence guarantees}
In this section, we show the convergence of the time-average number of active AI agents $x_{ji}(k)$ and active AI critics $y_{j \nu}(k)$ and feedback signals for agents $\Theta_j(k)$ and critics $\Theta^c_j(k)$. To do so, we use techniques from stochastic approximation. Results from Borkar \cite[Chapter 6]{Borkar2008} for multiple timescales are useful for this purpose. 

\subsection{AI agents}
For $i=1,2,\ldots,N$, $j=1,2,\ldots,m$, and $k\in \mathbb{N}$, let $\bm{X}_j(k) = (X_{j1}(k), \ldots, X_{jN}(k)) \in \{0,1\}^N$ and $\mathbf{x}_j(k) \in [0,1]^N$ denote the vectors with entries $X_{ji}(k)$ and $x_{ji}(k)$, respectively. Let $\bm{\sigma}_j(\Theta_j(k),\mathbf{x}(k))$ denote the vector of probabilities of task modality $j$ with entries $\sigma_{ji}(\Theta_j(k), x_{1i}(k), \ldots,x_{mi}(k))$.
For modalities $j=1,2, \ldots, m$, we reformulate the time-average number of active AI agents $\mathbf{x}_j(k)$, as follows:
	\begin{align} \label{eq:avg-y}
	\mathbf{x}_j(k+1) = \frac{k}{k+1} \mathbf{x}_j(k) + \frac{1}{k+1} \bm{X}_j(k+1).
	\end{align}
	Which is reformulated further as follows:
	\begin{align} \label{eq:avg_y_m}
	\mathbf{x}_j(k+1)  \nonumber \\  = \mathbf{x}_j(k) + &\frac{1}{k+1} \left [ \left( \bm{\sigma}_{j}(\Theta_j(k),\mathbf{x}(k)) - \mathbf{x}_j(k) \right)  \right. \\  \nonumber &+ \left. \left( \bm{X}_j(k+1) - \bm{\sigma}_{j}(\Theta_j(k),\mathbf{x}(k)) \right) \right].
	\end{align} 
Let $\left( \bm{X}_j(k+1) - \bm{\sigma}_{j}(\Theta_j(k),\mathbf{x}(k)) \right)$ be denoted by $\mathbf{M}_j(k+1)$, and the step-size $\frac{1}{k+1}$ be denoted by $\zeta(k)$, for $k \in \mathbb{N}$. Also, let $\left( \bm{\sigma}_{j}(\Theta_j(k),\mathbf{x}(k)) - \mathbf{x}_j(k) \right)$ be denoted by $\omega_j(\mathbf{x}_j(k))$. After replacing these values in \eqref{eq:avg_y_m}, we obtain 

\begin{align} \label{eq:avg_y_m1}
&\mathbf{x}_j(k+1) = \mathbf{x}_j(k) + \zeta(k) \left [ \omega_j(\mathbf{x}_j(k)) + \mathbf{M}_j(k+1) \right].
\end{align} 

Here, for a fixed $j$, $\{\mathbf{M}_j(k)\}$ is a {\em martingale difference sequence} with respect to a $\sigma$-algebra generated by the events up to time step $k$.

Let us assume two stochastic approximation systems with positive decreasing step sizes $\zeta(k)$ and $\eta(k)$. One step size converges to $0$ at a faster rate than the other step size. For example, for step sizes $\zeta(k)=\frac{1}{k+1}$ and $\eta(k)=\frac{1}{(k+1)^{3/4}}$, step size $\zeta(k)$ converges to $0$ at a faster rate than $\zeta(k)$. Moreover, for the joint analysis of the two systems, we call the system with step size $\zeta(k)$ as a {\em fast transient}, and the system with step size $\eta(k)$ as a {\em slow component}. 
We now state the following assumption and convergence result for multi-time scales with decreasing step sizes from \cite[Chapter 6]{Borkar2008}. Convergence results with fixed exponent multi-time scale decreasing step sizes for models with only agents is found at~\cite{Alam2023_CDC}.
\begin{assumption} \label{Ass:Lipschitz-St2}
Let $\mathbf{x}, \mathbf{z} \in \mathbb{R}_+^N$, and let the maps $\omega:\mathbb{R}^N \times \mathbb{R}^N \to \mathbb{R}^N$, and $h:\mathbb{R}^N \times \mathbb{R}^N \to \mathbb{R}^N$. For constant initial states $\mathbf{x}(0)$ and $\mathbf{z}(0)$, let systems $\mathbf{x}(k)$ and $\mathbf{z}(k)$ be formulated as follows
\begin{align} \label{eq:avg_x6.1}
& \mathbf{x}(k+1) = \mathbf{x}(k) + \zeta(k) \big [ \omega(\mathbf{x}(k), \mathbf{z}(k)) + \mathbf{M}_1(k+1) \big],
\end{align} 
\begin{align} \label{eq:avg_z6.2}
& \mathbf{z}(k+1) = \mathbf{z}(k) + \eta(k) \big [ h(\mathbf{x}(k), \mathbf{z}(k)) + \mathbf{M}_2(k+1) \big].
\end{align}
 We have the following assumptions.
 \begin{itemize}
		\item[(i)] The maps $\omega$ and $h$ are Lipschitz continuous.
		\item[(ii)] Step sizes $\{\zeta(k)\}_{k \in \mathbb{N}}$ and $\{\eta(k)\}_{k \in \mathbb{N}}$ are such that the following are satisfied:
		\begin{align*}
		&\lim_{k\to \infty} \zeta(k) =  \lim_{k\to \infty} \eta(k) = 0,
		\\
		&\sum_{\ell=0}^{\infty} \zeta(\ell) = \infty, \sum_{\ell=0}^{\infty} \eta(\ell) = \infty,
		\\
		&\sum_{\ell=0}^{\infty} \zeta(\ell)^2 + \sum_{\ell=0}^{\infty} \eta(\ell)^2 < \infty, \text{ and }
		\\
		& \lim_{k \to \infty}\frac{\zeta(k)}{\eta(k)} = 0.
		\end{align*}
		Notice that $\mathbf{x}(k)$ is a fast transient, and $\mathbf{z}(k)$ is a slow component. Furthermore, for the joint analysis, we can consider $\mathbf{z}(k)$ as quasi-static that is almost a constant for a large $k$.
		\item[(iii)] $\{\mathbf{M}_1(k)\}_{k \in \mathbb{N}}$ and $\{\mathbf{M}_2(k)\}_{k \in \mathbb{N}}$ are martingale difference sequences. Let $\mathcal{F}_k$ be a $\sigma$-algebra generated by the events up to time step $k$. For $k \in \mathbb{N}$, we have the following expectation:
		\begin{align*}
		& \mathbb{E} \left( \mathbf{M}_{1}(k+1) \mid \mathcal{F}_k \right) = 0
	\textnormal{ and }	
		 \mathbb{E} \left( \mathbf{M}_2(k+1) \mid \mathcal{F}_k \right) = 0.
		\end{align*}
		\item[(iv)] $\sup_{k} \left(\norm{\mathbf{x}(k)}_1 + \norm{\mathbf{z}(k)}_1 \right) < \infty$ almost surely.
	\end{itemize}
\end{assumption}

\begin{theorem}[Convergence with multi-time-scale step sizes]\cite[Chapter 6]{Borkar2008}  \label{Thm:Borkar-C6}
	Let $\mathbf{x}, \mathbf{z} \in \mathbb{R}_+^N$, and let the maps $\omega:\mathbb{R}^N \times \mathbb{R}^N \to \mathbb{R}^N$, and $h:\mathbb{R}^N \times \mathbb{R}^N \to \mathbb{R}^N$. For constant initial states $\mathbf{x}(0)$ and $\mathbf{z}(0)$, let $\mathbf{x}(k)$ be formulated as in \eqref{eq:avg_x6.1} and $\mathbf{z}(k)$ be formulated as in \eqref{eq:avg_z6.2},
	and Assumptions \ref{Ass:Lipschitz-St2} (i) to (iv) are satisfied, then $\left(\mathbf{x}(k), \mathbf{z}(k)\right)$ converges almost surely.
\end{theorem}
We now proceed to prove the convergence results for multimodalities. Hoeffding's inequality result presented below is useful to show the upper bound of the feedback signal $\Theta_j(k)$ for AI agents.
\begin{theorem} [Hoeffding's inequality \cite{Hoeffding1963}] \label{th_Hoeffding}
	For $i=1,2,\ldots,N$, let $X_i$ be the independent Bernoulli random variables, and let $\mathbb{E} (X_i)$ be their expectations, then for any $\epsilon >0$, we have
	\begin{align}  \label{Hoeff_1}
	\mathbb{P} \left(  \Big | \sum_{i=1}^N X_i - \mathbb{E} (\sum_{i=1}^N X_i) \Big | \geq \epsilon \right ) \leq 2\exp{(-2\epsilon^2/N)}.  
	\end{align} 
\end{theorem}

We state the following result on upper bound of the feedback signal for AI agents based on Hoeffding's inequality. 
	\begin{lemma} \label{lem:sup-Omega}
		For a fixed $j$, let $\Theta_j(0) = \Theta_{j0}>0$, and let $\Theta_j(k)$ be as in \eqref{omega_b1}. Let $\eta_{0_j}>0$, $\alpha_j>0$, and $\beta_j \in (0.5,1)$, and let step size $\eta_j(k) = \frac{\eta_{0_j}}{(\alpha_j k+1)^{\beta_j}}$. For any $\epsilon>0$, we have $\mathbb{P} \left(\vert (\Theta_j(k+1) - \mathbb{E}(\Theta_j(k+1) \vert \Theta_j(k)) \vert \geq \epsilon \right)  \leq 2\exp \left(\frac{-2 ((\alpha_j k+1)^{\beta_j}\epsilon)^2}{\eta_{0_j}^2 N} \right)$. 
	\end{lemma}
	
	\begin{proof} 
		From \eqref{omega_b1}, we have
		\begin{align*} 
		\Theta_j(k+1) = \Theta_j(k) - \eta_j(k)  \left( \sum_{i=1}^{N} X_{ji}(k+1) -\mathcal{N}_j \right).
		\end{align*}
		Taking the expectation, we obtain as follows:
		\begin{align*}   
		\mathbb{E}(\Theta_j(k+1) \vert \Theta_j(k)) &= \Theta_j(k) - \eta_j(k) \mathbb{E}\left( \sum_{i=1}^{N} X_{ji}(k+1)\right) \\ &+ \eta_j(k) \mathcal{N}_j.
		\end{align*}
		Furthermore, we obtain 
		\begin{align*} 
		&\Theta_j(k+1) - \mathbb{E}\Big(\Theta_j(k+1)\vert \Theta_j(k) \Big) \\&= \eta_j(k) \mathbb{E}\Big(\sum_{i=1}^{N} X_{ji}(k+1) \Big) - \eta_j(k) \sum_{i=1}^{N} X_{ji}(k+1).
		\end{align*}
From Hoeffding's inequality (refer Theorem \ref{th_Hoeffding}), we obtain
\begin{align*} 
    &\mathbb{P} \left(\left \vert \Theta_j(k+1) - \mathbb{E}(\Theta_j(k+1) \vert \Theta_j(k)) \right \vert  \geq \epsilon \right) \\
    &= \mathbb{P} \left( \left \vert \eta_j(k) \left( \mathbb{E}\left(\sum_{i=1}^{N} X_{ji}(k+1) \right) - \sum_{i=1}^{N} X_{ji}(k+1) \right) \right \vert \geq \epsilon \right) \\
	&= \mathbb{P} \left( \left \vert \mathbb{E} \left(\sum_{i=1}^{N} X_{ji}(k+1) \right) -\sum_{i=1}^{N} X_{ji}(k+1) \right \vert \geq \frac{\epsilon}{\eta_j(k)}  \right) \\
	&\leq 2\exp \left(\frac{-2 ((\alpha_j k+1)^{\beta_j}\epsilon)^2}{\eta_{0_j}^2 N} \right).
\end{align*}
\end{proof}

 We present the following result on bounds for feedback signals $\Theta_j(k)$ and agents' time-average active states $x_{ji}(k)$.
	\begin{lemma} \label{lem:sup-Omega-y}
		For a fixed modality $j$, let $\Theta_j(0) = \Theta_{j0}>0$, and let $\Theta_j(k)$ be defined as in \eqref{omega_b1}, we have $\sup_k \left ( \norm{\mathbf{x}_j(k)}_1 + \norm{\Theta_j(k)}_1 \right) < \infty$ almost surely.
	\end{lemma}
	\begin{proof}
		With bounded $x_{ji}(k) \in [0,1]$ and $\Theta_j(k)$ (from Lemma \ref{lem:sup-Omega}), it is   straightforward to obtain the result.
	\end{proof}

We now show that, for modality $j=1,2,\ldots,m$, the sequence $\left\{\left(\mathbf{x}_j(k), \Theta_j(k) \right)\right\}_{k \in \mathbb{N}}$ converges almost surely.
	\begin{theorem} \label{th:V.6}
		For any modality $j$, and fixed constant initial states $\mathbf{x}_j(0)$ and $\Theta_j(0)$, the sequence $\left\{\left(\mathbf{x}_j(k), \Theta_j(k) \right)\right\}_{k \in \mathbb{N}}$ converges, almost surely. 
	\end{theorem}
	\begin{proof}
		
Recall that the AI agent's active state $X_{ji}(k+1)$ is updated as in \eqref{eq:alloc-update} and the probability $\sigma_{ji}(\Theta_j(k),x_{1i}(k), \ldots, x_{mi}(k))$ is calculated as in \eqref{eq:prob}, for all $i,j$, and $k$. 
Moreover, the time-average AI agent's active state $x_{ji}(k+1)$ is formulated as in \eqref{eq:avg-y}, recall that $\mathbf{x}_j(k+1) = (x_{j1}(k+1), \ldots, x_{jN}(k+1))$. %
Let the map $\omega_j: \R \times (\R^{N})^m \to \R^N$, and let $\big(\bm{\sigma}_{j}(\Theta_j(k),\mathbf{x}(k)) - \mathbf{x}_j(k) \big)$ be denoted by $\omega_j(\Theta_{j}(k), \mathbf{x}(k))$. Also, let $\big (\bm{X}_j(k+1) - \bm{\sigma}_{j}(\Theta_j(k),\mathbf{x}(k)) \big)$ be denoted by $\mathbf{M}_{j,1}(k+1)$. Analogous to \eqref{eq:avg_y_m1}, we obtain:
\begin{align} \label{eq:V.6-avg}
\mathbf{x}_j(k+1) = \mathbf{x}_j(k) + \zeta(k) \left [\omega_j(\Theta_{j}(k), \mathbf{x}(k)) + \mathbf{M}_{j,1}(k+1) \right],
\end{align}
		where $\zeta(k)=\frac{1}{k+1}$.
Reformulating the feedback signal $\Theta_j(k)$ for AI agents (refer the definition \eqref{omega_b1}), we obtain as follows:
		\begin{align} \label{eq:V.6-Omega}
		&\Theta_{j}(k+1) \nonumber = \Theta_{j}(k)   \\&+ \frac{\eta_{0_j}}{(\alpha_j k+1)^{\beta_j}}  \left ( \left (\mathcal{N}_j -  \sum_{i=1}^N \sigma_{ji}(\Theta_j(k),x_{1i}(k), \ldots, x_{mi}(k)) \right) \right. \nonumber \\& \left. +  \left ( \sum_{i=1}^N \sigma_{ji}(\Theta_j(k),x_{1i}(k), \ldots, x_{mi}(k)) - \sum_{i=1}^N X_{ji}(k+1) \right ) \right ).
		\end{align}
		For $j=1,2,\ldots,m$, let the map $h_j : \mathbb{R} \times (\mathbb{R}^{N})^m \to \mathbb{R}$, and let $\left (\mathcal{N}_j -  \sum_{i=1}^N \sigma_{ji}(\Theta_j(k),x_{1i}(k), \ldots, x_{mi}(k)) \right)$ be denoted by $h_j(\Theta_{j}(k),\mathbf{x}(k))$. 

		Let 
         $\left ( \sum_{i=1}^N \sigma_{ji}(\Theta_j(k),x_{1i}(k), \ldots, x_{mi}(k)) - \sum_{i=1}^N X_{ji}(k+1) \right )$ be denoted by $M_{j,2}(k+1)$. Recall that we denote the step-size $\frac{\eta_{0_j}}{(\alpha_j k+1)^{\beta_j}}$ by $\eta_j(k)$.
		Then from \eqref{eq:V.6-Omega}, we obtain:
		\begin{align} \label{eq:V.6-Omega2}
		&\Theta_{j}(k+1)  \nonumber \\&= \Theta_{j}(k) +  \eta_j(k)\Big( h_j(\Theta_{j}(k), \mathbf{x}(k)) + M_{j,2}(k+1) \Big).
		\end{align}
		Assumption \ref{Ass:Lipschitz-St2} (i) is satisfied, as the maps $\omega_j$ and $h_j$ are Lipschitz continuous, for $j=1,2, \ldots,m$.
		We have step-size $\zeta(k)=\frac{1}{k+1}$, and step-size $\eta_j(k) = \frac{\eta_{0_j}}{(\alpha_j k+1)^{\beta_j}}$. Here, $\eta_{0_j}>0$ and $\alpha_j>0$ are small real valued numbers, and $\beta_j \in (1/2,1)$. With simple mathematical manipulation, we obtain as follows:
		\begin{align*}
		&\lim_{k\to \infty} \zeta(k) = 0, \lim_{k\to \infty} \eta_j(k) = 0,
		\\
		&\sum_{\ell=0}^{\infty} \zeta(\ell) = \infty, \sum_{\ell=0}^{\infty} \eta_j(\ell) = \infty,
		\\
		&\sum_{\ell=0}^{\infty} \zeta(\ell)^2 + \sum_{\ell=0}^{\infty} \eta_j(\ell)^2 < \infty, \text{ and }
		\\
		& \frac{\zeta(k)}{\eta_j(k)} \to 0, \text{ when } k \to \infty.
		\end{align*}
		Note that $\zeta(k)$ is a fast transient and $\eta_j(k)$ is a slow component. These satisfy Assumption \ref{Ass:Lipschitz-St2} (ii).
		
		Recall that $\mathcal{F}_k$ denotes a $\sigma$-algebra generated by the events up to time step $k$, then we obtain the following expectation:
		\begin{align*}
		\mathbb{E} \left( \mathbf{M}_{j,1}(k+1) \mid \mathcal{F}_k \right) = 0, \text { for } k \in \mathbb{N},
		\end{align*}
		and
		\begin{align*}
		\mathbb{E} \left( M_{j,2}(k+1) \mid \mathcal{F}_k \right) = 0,  \text { for } k \in \mathbb{N}.
		\end{align*}
		Thus, $\{\mathbf{M}_{j,1}(k)\}_{k \in \mathbb{N}}$ and $\{M_{j,2}(k)\}_{k \in \mathbb{N}}$ are martingale difference sequences that satisfy Assumption \ref{Ass:Lipschitz-St2} (iii).
		From Lemma \ref{lem:sup-Omega-y}, for a fixed $j$, we have $\sup_k \left ( \norm{\mathbf{x}_j(k)}_1 + \norm{\Theta_j(k)}_1 \right) < \infty$ almost surely that satisfies Assumption \ref{Ass:Lipschitz-St2} (iv).
		Thus for all $j$, from Theorem \ref{Thm:Borkar-C6}, we conclude that $\left(\Theta_j(k), \mathbf{x}_j(k) \right)$ converges, almost surely.
	\end{proof}
\subsection{AI critics}
We state the following result for the convergence of the time-average number of active states $y_{j\nu}(k)$ of AI critics and the feedback signals for critics $\Theta^c_j(k)$. The convergence proof is similar to the AI agents presented in the previous subsection; for the sake of exposition, we present here the results for AI critics.

For modalities $j=1,2, \ldots, m$, we reformulate the time-average number of active AI critics $\mathbf{y}_j(k+1) = (y_{j1}(k+1), \ldots, y_{jC_j}(k+1))$, as follows:
	\begin{align} \label{eq:avg-y_critic}
	\mathbf{y}_j(k+1) = \frac{k}{k+1} \mathbf{y}_j(k) + \frac{1}{k+1} \bm{y}_j(k+1).
	\end{align}
    Let $\bm{\sigma}^c_j(\Theta^c_j(k),\mathbf{y}_j(k))$ denote the vector of probabilities of task modality $j$ with entries $\sigma^c_{j\nu}(\Theta^c_j(k), y_{j\nu}(k)$.
	We reformulate \eqref{eq:avg-y_critic} as follows:
	\begin{align} \label{eq:avg_y_m_critic}
	\mathbf{y}_j(k+1)  \nonumber \\  = \mathbf{y}_j(k) + &\frac{1}{k+1} \left [ \left( \bm{\sigma}^c_{j}(\Theta^c_j(k),\mathbf{y}_j(k)) - \mathbf{y}_j(k) \right)  \right. \\  \nonumber &+ \left. \left( \bm{y}_j(k+1) - \bm{\sigma}^c_{j}(\Theta^c_j(k),\mathbf{y}_j(k)) \right) \right].
	\end{align} 
Let $\left( \bm{y}_j(k+1) - \bm{\sigma}^c_{j}(\Theta^c_j(k),\mathbf{y}_j(k)) \right)$ be denoted by $\mathbf{M}^c_{j,1}(k+1)$; it is a  martingale difference sequence. Recall that the step-size $\frac{1}{k+1}$ is denoted by $\zeta(k)$, for $k \in \mathbb{N}$. Let the map $\omega^c_j:\mathbb{R} \times \mathbb{R}^{C_j} \to \mathbb{R}$, and let $\left( \bm{\sigma}^c_{j}(\Theta^c_j(k),\mathbf{y}_j(k)) - \mathbf{y}_j(k) \right)$ be denoted by $\omega^c_j(\Theta^c_j(k), \mathbf{y}_j(k))$. After replacing these values in \eqref{eq:avg_y_m_critic}, we obtain 

\begin{align} \label{eq:avg_y_m1_critic}
&\mathbf{y}_j(k+1) = \mathbf{y}_j(k) + \zeta(k) \left [ \omega^c_j(\Theta^c_j(k), \mathbf{y}_j(k)) + \mathbf{M}^c_{j,1}(k+1) \right].
\end{align} 

We reformulate the feedback signal $\Theta^c_j(k)$ for AI critics defined in \eqref{omega_b2} as follows: 
		\begin{align} \label{eq:V.6-Omega_critic}
		&\Theta^c_{j}(k+1) \nonumber \\ &= \Theta^c_{j}(k) +  \eta^c_j(k)  \left ( \left (\mathcal{M}_j -  \sum_{\nu=1}^{C_j} \sigma^c_{j\nu}(\Theta^c_j(k),y_{j\nu}(k)) \right) \right. \nonumber \\& \left. +  \left ( \sum_{\nu=1}^{C_j} \sigma^c_{j\nu}(\Theta^c_j(k),y_{j\nu}(k)) - \sum_{\nu=1}^{C_j} Y_{j\nu}(k+1) \right ) \right ).
		\end{align}
		For $j=1,2,\ldots,m$, let the map $h^c_j : \mathbb{R} \times \mathbb{R}^{C_j} \to \mathbb{R}$, and let $\left (\mathcal{M}_j -  \sum_{\nu=1}^{C_j} \sigma^c_{j\nu}(\Theta^c_j(k),y_{j\nu}(k)) \right)$ be denoted by $h^c_j(\Theta^c_{j}(k),\mathbf{y}_j(k))$. 
Furthermore, let $\left ( \sum_{\nu=1}^{C_j} \sigma^c_{j\nu}(\Theta^c_j(k),y_{j\nu}(k)) - \sum_{\nu=1}^{C_j} Y_{j\nu}(k+1) \right )$ be denoted by $M^c_{j,2}(k+1)$. Recall that the step-size $\frac{\eta^c_{0_j}}{(\alpha^c_j k+1)^{\beta^c_j}}$ is denoted by $\eta^c_j(k)$.
From \eqref{eq:V.6-Omega_critic}, we obtain:
\begin{align} \label{eq:V.6-Omega2_critic}
\Theta^c_{j}(k+1)  = \Theta^c_{j}(k) +  \eta^c_j(k)\Big( h^c_j(\Theta^c_{j}(k), \mathbf{y}_j(k)) + M^c_{j,2}(k+1) \Big).
\end{align}

The result on upper bound of the feedback signal for AI critics is presented as follows. 
	\begin{lemma} \label{lem:sup-Omega_critic}
		For a fixed modality $j$, let $\Theta^c_j(0) = \Theta^c_{j_0}>0$ be the initial value of feedback signal for agents completing tasks of modality $j$, and let $\Theta^c_j(k)$ be as in \eqref{omega_b2}. Let $\eta^c_{0_j}>0$, $\alpha^c_j>0$, and $\beta^c_j \in (0.5,1)$, and let step size $\eta^c_j(k) = \frac{\eta^c_{0_j}}{(\alpha^c_j k+1)^{\beta^c_j}}$. For any $\epsilon^c>0$, we have $\mathbb{P} \left(\vert (\Theta^c_j(k+1) - \mathbb{E}(\Theta^c_j(k+1) \vert \Theta^c_j(k)) \vert \geq \epsilon^c \right)  \leq 2\exp \left(\frac{-2 ((\alpha^c_j k+1)^{\beta^c_j}{\epsilon^c})^2}{{(\eta^c_{0_j})}^2 C_j} \right)$. 
	\end{lemma}
	
We now present the convergence result of the time-average number of active AI critics $\mathbf{y}_j(k)$ and the feedback signals for AI critics $\Theta^c_j(k)$.
\begin{theorem}[Convergence results of AI critics] \label{th:V.7_critic}
		For any task modality $j$, and fixed constant initial values $\mathbf{y}_j(0)$ and $\Theta^c_j(0)$, the sequence $\left\{\left(\mathbf{y}_j(k), \Theta^c_j(k) \right)\right\}_{k \in \mathbb{N}}$ converges, almost surely. 
	\end{theorem}
\begin{proof}
As stated previously, the step size $\zeta(k)$ is fast transient and $\eta^c_j(k)$ is a slow component. Furthermore, $\mathbf{M}^c_{j,1}(k+1)$ and $M^c_{j,2}(k+1)$ are martingale difference sequences, for the given sigma algebra $\mathcal{F}_k$; refer \eqref{eq:avg_y_m1_critic} and \eqref{eq:V.6-Omega2_critic}, respectively. Following the steps similar to AI agents as in the proof of Theorem \ref{th:V.6}, we show that Assumption \ref{Ass:Lipschitz-St2} is satisfied.
\end{proof}

Finally, we make the following remark about the communication complexity of the model.
\begin{remark}[Bounds on Communication overhead]
The AI agents and critics do not share cost function or their (partial) derivatives with other AI agents or critics in the system, the central server broadcasts the feedback signals $\Theta_j(k)$ and  $\Theta^c_j(k)$ at each time step; because of the broadcast of the feedback signals, the system incurs a little communication overhead, of the order of $\mathcal{O}(m)$, for $m$ modalities. Moreover, in a double precision floating point format, if the system completes the tasks of $m$ modalities, the communication overhead on the system will be $2\times 64 m$ bits per time step. The communication overhead at a time step will be $128 m$ bits, which is a little, and it does not depend on the number of AI agents and critics in the system. 
\end{remark}
 
\section{Experimental results: Fault detection and cause analysis in a multimodal network telemetry system}  \label{bin_applications}

We evaluate our model with an open-source network telemetry dataset available at \cite{Putina2021} for fault detection,  severity, and cause analysis. We chose dataset $4$ that contains tabular data in a comma-separated values (CSV) format and log files for port state changes. The data has three types of port state changes: Optical transceiver pull and reinsert, admin state change (down and up), and link state changes (plug-in and plug-out of fiber optics cable). For fault detection, developed model is evaluated by two approaches: AI agents  enabled with classical machine learning techniques (e.g., XG Boosting), and AI agents and critics enabled with large language models. We observe that agents enabled with XG Boosting provide faster results and higher accuracy, F1-score, precision, and recall than AI agents and critics enabled with LLMs.

As described earlier, the CSV file of network telemetry dataset $4$ of \cite{Putina2021} has $99$ columns, which include numeric and non-numeric data. Data are pre-processed before being accessible to AI agents and critics.  The pre-processing step is presented in the following subsection.

\subsection{Pre-processing of network telemetry data} \label{preprocess_CSV}
The time column in the CSV file is in Unix nanosecond format, whereas the dates in log and event files are in Pacific Daylight time (PDT). To unify the formats, Unix nanosecond time is converted into Pacific Daylight time (PDT). Furthermore, all the numeric columns with zero variance are removed because they do not affect fault analysis. 
Using the time duration of fault events presented in the event file of the dataset 4 of \cite{Putina2021}, we extract the fault data and add a fault column in the CSV file. An equal number of no-fault entries is added to the CSV file; the entries are labeled as ``1'' for fault and ``0'' for no fault, we call the dataset a {\em balanced dataset}. Furthermore, these entries are randomized.

{\bf Fault detector agents enabled with the XG Boosting:} For the XG Boosting, the balanced dataset with fault column (CSV file) is split into two parts---one is used for training and another for testing purposes. XG Boosting model is trained with 2/3 of the records, and the remaining records (test files) are split into smaller files. The XG Boosting trained model is saved and accessed by fault detector agents. At each iteration, active agents access a CSV test file to predict the faults. The resulting fault-classified data is saved as a CSV file. 

{\bf Fault detector agents and critics enabled with large language models (LLMs):}
AI agents and critics can access different large language models (LLMs), such as Llama3.2, DeepSeek-R1, Mistral, etc., for response generation and evaluation. For LLM-enabled agents and critics, the CSV data are converted into text queries and saved as CSV files for efficient processing by the LLMs.

For the fault detection step with LLM-enabled agents and critics, the text queries CSV file is then converted into smaller chunks to pass to agents and critics for fault detection. At each iteration, an active agent processes a file and passes its responses to critics for feedback. An active critic then evaluates the agent's responses, checks their correctness, and sends feedback to the agent. The agent then considers the critic's feedback to improve its response; however, if the agent disagrees with the critic's evaluation, it does not change the answer.

The following is an example of queries for fault detection with LLM-enabled agents and critics.
 \begin{query}
Predict whether Device with `name': `Cisco-IOS-XR-infra-statsd-oper:infra-statistics/interfaces/interface/latest/generic-counters', `time': `2017-09-01 10:36:28.566000', `EncodingPath': `Cisco-IOS-XR-infra-statsd-oper:infra-statistics/interfaces/interface/latest/generic-counters', `Producer': `leaf5', `bytes-received': `862661866283767.0', `bytes-sent': `596481617153767.0', `carrier-transitions': `0.0', `crc-errors': '0.0', `input-drops': `0.0', `input-errors': `0.0', `input-ignored-packets': `0.0', `input-queue-drops': `0.0', `interface-name': `HundredGigE0/0/0/17', `output-buffer-failures': `0.0', `output-drops': `0.0', `output-errors': `0.0', `output-queue-drops': `0.0', `packets-received': `254910632957.0', `packets-sent': `241408351439.0' has fault or not.
\end{query}

{\bf Fault severity and cause analysis:}
In the fault severity and cause analysis step, we use LLM-enabled agents and critics to process fault queries from both approaches---the XG Boosting and LLM-based. Recall that in the fault detection step with the XG Boosting, the fault-classified data is saved in CSV files, which are converted into text queries with predicted faults as a column. Whereas, for LLM-enabled fault detection, the CSV network telemetry data is pre-processed to create text queries and accessed by agents and critics to predict faults (`Fault', `No Fault'), the predicted value is added as a column in the queries and saved as CSV files. The queries with predicted `Fault' are saved in another CSV files based on processed query chunks. These query data files are accessible to AI agents for fault severity and cause analysis. The AI agents first classify the faults as `Critical' or `Non-critical' with access to the network telemetry log files (stored in ChromaDB) and provide a brief cause of the fault. The AI agent's response is then sent to AI critics for evaluation. An active critic evaluates the correctness of the agent's fault severity classification and cause summary and sends feedback to the agent for further improvement. The critics also have access to network telemetry log files through ChromaDB instances. After receiving the feedback, the agent improves its response. The responses are saved in CSV files. The `Critical' flagged faults can be notified to network administrator. 

\subsection{Setup}
In the experiment, we consider $9$ agents participating in completing tasks of two modalities. For example, the fault detection step uses network telemetry CSV data, whereas the fault severity and cause analysis step uses fault queries CSV data and network telemetry log files. However, it can be easily implemented for more than two modalities. We first present convergence results for a generic case for two modalities. Then, we present specific fault detection, severity, and cause analysis results with agents enabled by classical machine learning and generative AI's foundation model. For modality $1$, we chose the desired number of agents completing the task to be $6$; we also call it the {\em capacity}, denoted by $\mathcal{N}_1$. Analogously, for modality $2$, the desired number of agents completing the task is $7$, the capacity $\mathcal{N}_2=7$. Moreover, we chose $3$ critics for each modality to evaluate the agents' responses. At every time step, $2$ critics for each modality are desired to be active to evaluate the agents' responses. 

Agents and critics access generative AI's foundation models or classical machine learning models. In the experiment, we use open-source frameworks. We use Ollama to deploy foundation models on our local machine and Langchain to facilitate the integration of various foundation models such as Llama3.2, DeepSeek-R1, Mistral, Llava:7b, etcetera. We implement the algorithms on a machine with Intel Core $i9$, CPU $2.20$ GHz, RAM $32$ GB, and NVIDIA's GPU GeForce RTX $4080$.
  
In the model, a central server collaborates with the AI agents and critics and knows the number of active AI agents and critics at a time step. It also knows the capacities of agents and critics for both modalities. The central server updates and sends feedback signals to agents and critics for each modality at a time step.

The AI agents have the following cost functions coupled through their time-average active states for both modalities. In the cost function, $a_i$, $b_i$, $c_i$, and $d_i$ are the uniformly distributed random variables, for $i=1,2,\ldots,N$. We choose $a_i \in (0,20)$, $b_i \in (0,15)$, $c_i \in (0,20)$, and $d_i \in (0,30)$, for $i=1,2,\ldots,N$, the cost function of AI agents presented as follows:

  \begin{align}
      &f_i(x_{1i}, x_{2i}) = \frac{1}{2} a_{i}x_{1i}^2+\frac{3}{4} c_{i} x_{1i}^4 + \frac{1}{2} b_{i} x_{2i}^2 + \frac{1}{2} d_i x_{2i}^4, \\&\textnormal{ for } i=1,2,\ldots,N. \nonumber
  \end{align}

  Critic $\nu$ has the following cost functions that depend on its time-average active states for the respective task modalities. In the following cost function, $a'_\nu$, $b'_\nu$, $c'_\ell$, and $d'_\ell$ are the uniformly distributed random variables, for $\nu=1,2,\ldots,C_1$ and $\ell=1,2,\ldots,C_2$. We choose $a'_\nu \in (0,10)$, $b'_\nu \in (0,6)$, $c'_\ell \in (0,10)$, and $d'_\ell \in (0,6)$.

The critics evaluating tasks of modality $1$ have the following cost function:
\begin{align}
    g_{1 \nu} (y_{1\nu}) &= a'_{\nu} y_{1\nu}^2 + b'_{\nu} y_{1\nu}^4,  \textnormal{ for } \nu=1,2,\ldots,C_1.   
\end{align}

The critics evaluating tasks of modality $2$ have the following cost function:
\begin{align}
    g_{2 \ell} (y_{2\ell}) &= \frac{1}{2} c'_{\ell} y_{2\ell}^2 + \frac{1}{2}  d'_{\ell} y_{2\ell}^4, \textnormal{ for } \ell=1,2,\ldots,C_2.    
\end{align}

The initial values of feedback signals $\Theta_1(0)$ and $\Theta_2(0)$ for AI agents  are initialized with $0.1$. Recall that the step size for the feedback signals for AI agents completing tasks of modality $j$ is $\eta_j(k) = \frac{\eta_{0_j}}{(\alpha_j k +1)^{\beta_j}}$, for task modalities $j=1,2$. For the two task modalities, we chose the same values, $\eta_{0_1} =\eta_{0_2} =  0.034$, $\alpha_1=\alpha_2=0.14$, and $\beta_1 = \beta_2 =0.51$. 

The initial values of feedback signals  $\Theta^c_1(0)$ and $\Theta^c_2(0)$ for AI critics are initialized with $0.33$. Recall that the step size for the feedback signals for critics evaluating tasks of modality $j$ is $\eta^c_j(k) = \frac{\eta^c_{0_j}}{(\alpha^c_j k +1)^{\beta^c_j}}$, $j=1,2$. We chose the same values for the parameters, $\eta^c_{0_1} = \eta^c_{0_2} =0.035$, $\alpha^c_1=\alpha^c_2=0.15$, and $\beta^c_1 = \beta^c_2 = 0.51$. The central server updates the feedback signals for AI agents and critics as in \eqref{omega_b1} and \eqref{omega_b2}, respectively. Notice that the AI critics' cost functions are not coupled and are independent for modalities $1$ and $2$. 
 
\subsection{Evaluations}
We now present the experimental results for two modalities for generic tasks for exposition. The results hold for the fault detection and severity and cause analysis example, presented later in this section. To compare the convergence results obtained from our algorithm, we use the CVX solver that solves Optimization problem \ref{prob:objective_fn} in a centralized way. The centralized optimal solution is considered the gold standard.

Figures \ref{fig:agent_avg_mod1} and  \ref{fig:agent_avg_mod2} illustrate the time-average active states of AI agents completing tasks of modality $1$ and $2$, respectively. The time average active states of AI agents $x_{1i}(k)$ and $x_{2i}(k)$ converge over time, close to their respective optimal values, $x^*_{1i}$ and $x^*_{2i}$. 
\begin{figure}
\centering
\begin{subfigure}{0.49\linewidth}
    \centering
    \includegraphics[width=\linewidth]{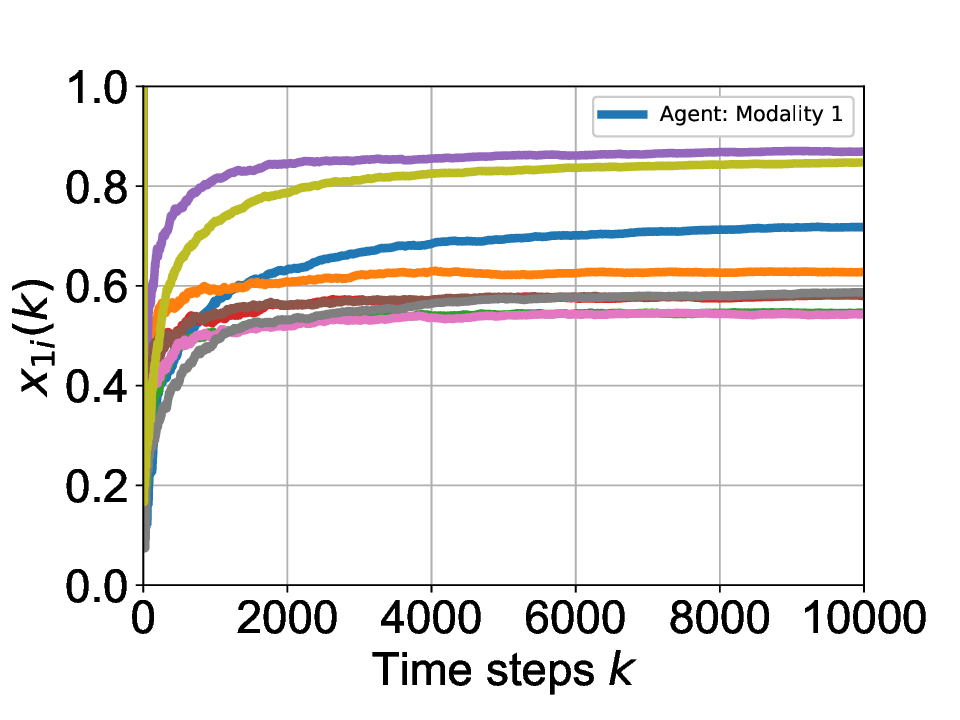}
    \caption{Modality $1$}
    \label{fig:agent_avg_mod1}
\end{subfigure}
\hfill
\begin{subfigure}{0.49\linewidth}
    \centering
    \includegraphics[width=\linewidth]{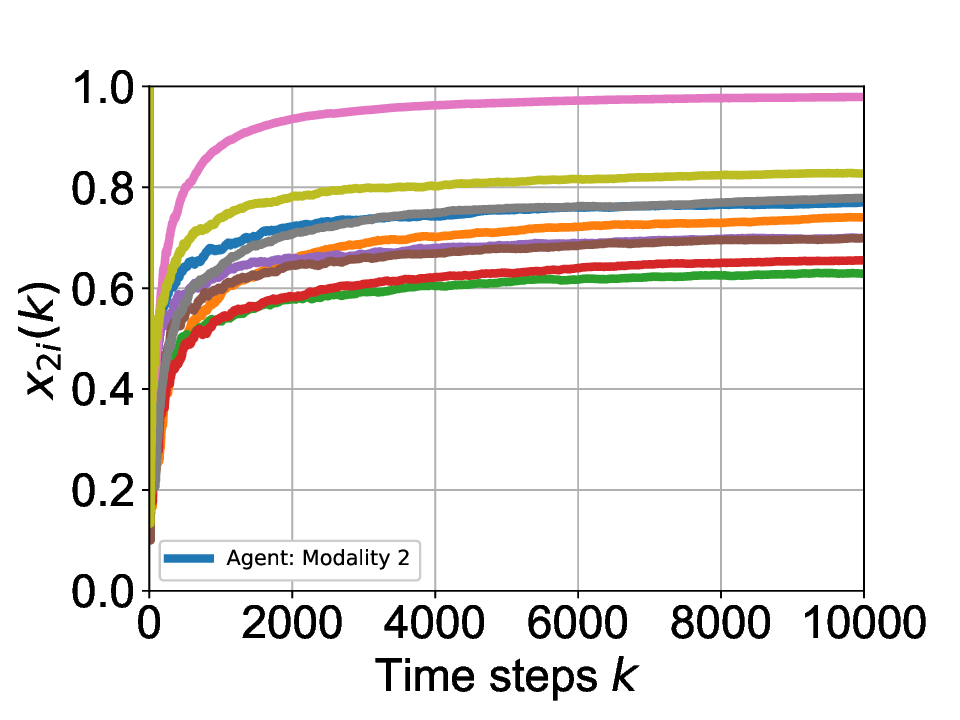}
    \caption{Modality $2$}
    \label{fig:agent_avg_mod2}
\end{subfigure}
\caption{The convergence of time-average active states $x_{1i}$ and $y_{2i}$ of AI agents performing tasks of (a) modality $1$ and (b) modality $2$, respectively, to close to their respective optimal values $x_{1i}^*$ and $x_{2i}^*$.}
\label{fig:agent_avg}
\end{figure}
Analogously, Figures \ref{fig:critic_avg_m1} and  \ref{fig:critic_avg_m2} illustrate the convergence of the time average active states of AI critics $y_{1i}(k)$ and $y_{2i}(k)$, close to their respective optimal values, $y_{1i}^*$ and $y_{2i}^*$. 
\begin{figure}
\centering
\begin{subfigure}{0.49\linewidth}
    \centering
    \includegraphics[width=\linewidth]{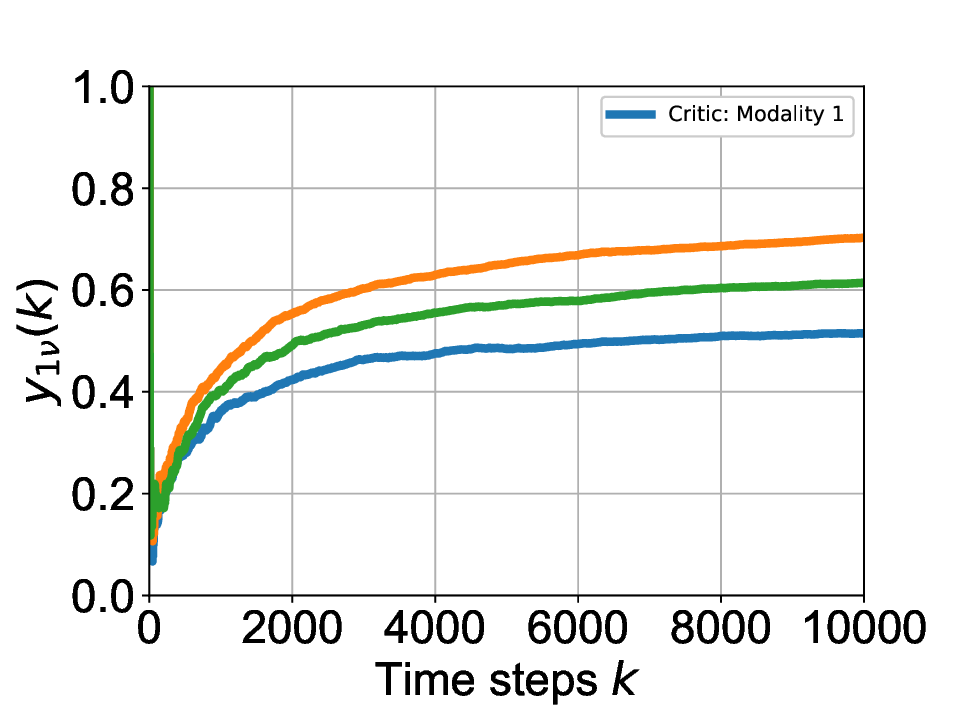}
    \caption{Modality $1$}
    \label{fig:critic_avg_m1}
\end{subfigure}
\hfill
\begin{subfigure}{0.49\linewidth}
    \centering
    \includegraphics[width=\linewidth]{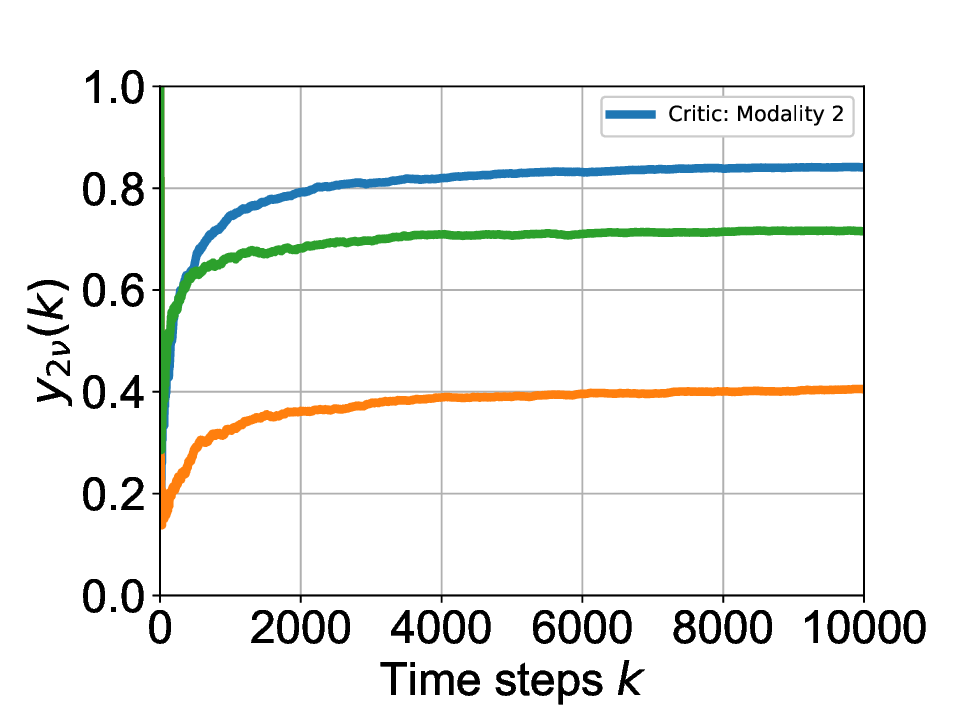}
    \caption{Modality $2$}
    \label{fig:critic_avg_m2}
\end{subfigure}
\caption{The convergence of time-average active states $y_{1\nu}$ and $y_{2\ell}$ of AI critics  evaluating tasks of (a) modality $1$ and (b) modality $2$, respectively, to close to their respective optimal values $y^*_{1\nu}$ and $y^*_{2\ell}$, $\nu=1,2,\ldots,C_1$, and $\ell=1,2,\ldots,C_2$.}
\label{fig:critic_avg}
\end{figure}

\begin{figure}
    \centering
   \includegraphics[width=0.49\linewidth]{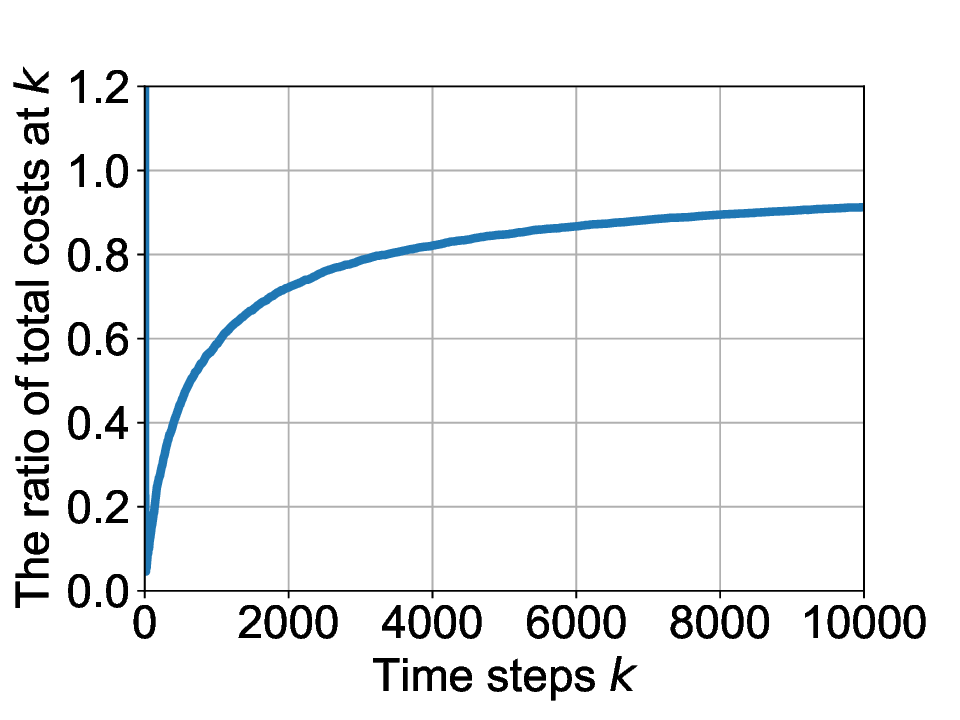}
    \caption{The evolution of the ratio of total cost obtained by the developed algorithm and the total optimal cost obtained by the CVX solver.}
    \label{fig:total_cost}
\end{figure}

\begin{figure}
    \centering
   \includegraphics[width=0.49\linewidth]{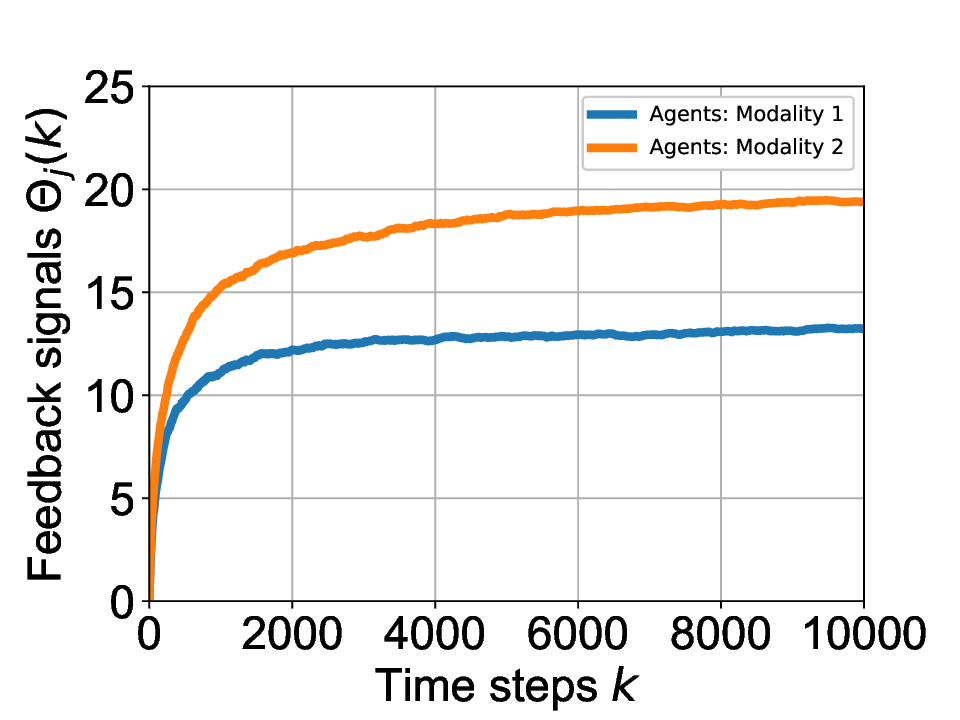}
    \caption{The evolution of feedback signals $\Theta_1(k)$ and $\Theta_2(k)$  broadcast by the central server for agents completing tasks of modalities $1$ and $2$, respectively.}
    \label{fig:feedback_signal_agents}
\end{figure}

\begin{figure}
    \centering
   \includegraphics[width=0.49\linewidth]{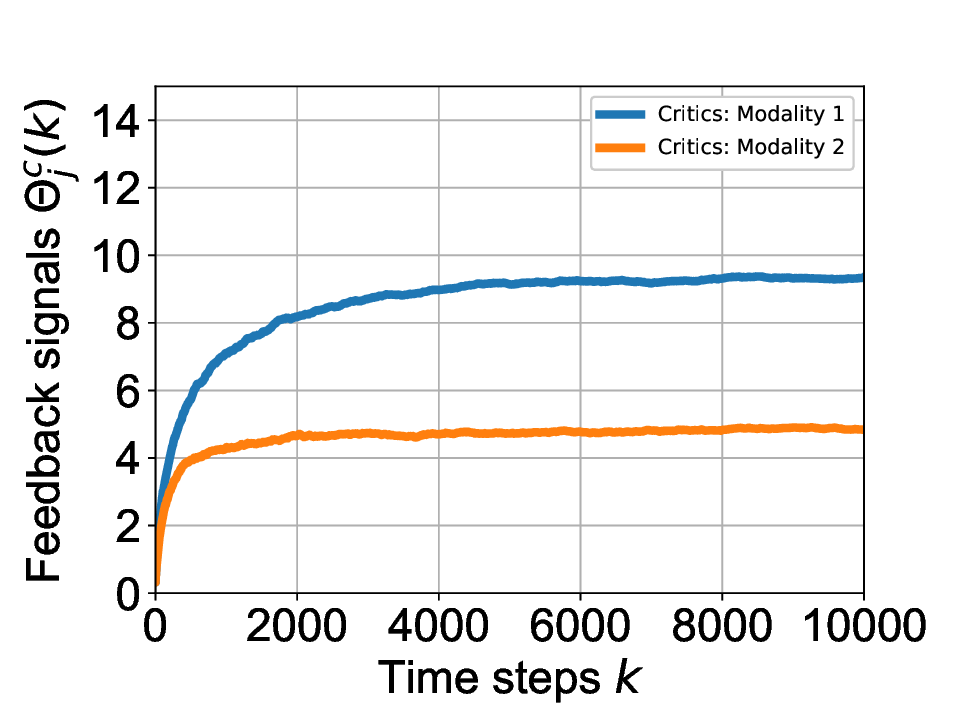}
    \caption{The evolution of feedback signals $\Theta^c_1(k)$ and $\Theta^c_2(k)$ broadcast by the central server for critics evaluating tasks of modalities $1$ and $2$, respectively.}
    \label{fig:feedback_signal_critics}
\end{figure}

\begin{table*}
\footnotesize
\setlength{\tabcolsep}{1.2pt}
\begin{tabular}{|c|c|p{1.6cm}|c|p{1.6cm}|}
        \hline
     
        \textbf{Models} & \textbf{Accuracy} & \textbf{Precision} & \textbf{Recall} & \textbf{F1-Score} 
        \\ \hline
           \hline
             Agents with XG Boosting & 0.9644 & 0.9397 & 0.9929 & 0.9656\\
             \hline
             Agents and Critics enabled with Llama3.2 (with  RAG) & 0.5882 & 0.8000 & 0.6154 & 0.6957\\
             \hline
            Agents and Critics enabled with Llama3.2 (without RAG) &  0.4612 & 0.4027 & 0.6648 & 0.5016 \\
\hline
    \end{tabular}
    \caption{Evaluation of the algorithms for fault detection on network telemetry dataset $4$ of \cite{Putina2021}: Agents are enabled with either classical machine learning techniques (e.g., XG Boosting) or large language models with (or without) retrieval augmented generation (RAG). The agents enabled with XG Boosting perform better than the large language model-enabled agents on the dataset.}
    \label{tab:fault_detection}
\end{table*}

\begin{table*}
\tiny
\begin{center}
\renewcommand{\arraystretch}{1}
\setlength{\tabcolsep}{2pt}
\begin{tabular}{|c| c |p{5.6cm} |c|p{5.6cm}|}
\hline
\textbf{Models} & \textbf{Actor} & \textbf{Query} & \textbf{Severity} & \textbf{Cause Summary} \\ 
\hline
\hline
Llama3.2 &  & \multirow{6}{5.6cm}{24. Predict whether Device with `name': `Cisco-IOS-XR-infra-statsd-oper:infra-statistics/interfaces/interface/latest/data-rate', `time': `2017-09-01 10:23:18.222000', `EncodingPath': `Cisco-IOS-XR-infra-statsd-oper:infra-statistics/interfaces/interface/latest/data-rate', `Producer': `leaf3', `bandwidth': `100000000.0', `input-data-rate': `10261058.0', `input-load': `26.0', `input-packet-rate': `166870.0', `interface-name': `HundredGigE0/0/0/5', `load-interval': `0.0', `output-data-rate': `10482276.0', `output-load': `26.0', `output-packet-rate': `449830.0', `peak-input-data-rate': `0.0', `peak-input-packet-rate': `0.0', `peak-output-data-rate': `0.0', `peak-output-packet-rate': `0.0', `reliability': `255.0' has a Critical fault or Non-critical fault.} & Non-critical & The device experienced a sudden decrease in input data rate due to a likely temporary network performance issue with the 'HundredGigE0/0/0/5' interface. However, it is stated that the interface is operational, indicating it's not a critical failure. \\ 
\cline{1-1}
\cline{4-5}
DeepSeek-R1 &  &  & Non-critical & Slight inefficiency due to output data rate increase relative to input with observed packet loss. \\ \cline{1-1}
\cline{4-5}
Mistral &  &  & Non-critical & The high levels of data rates and load on interface 'HundredGigE0/0/0/5' are within acceptable bounds for a production network, indicating normal operation. However, it is advised to monitor these values over time to ensure they do not exceed thresholds, indicating potential network congestion or performance degradation. \\ \cline{1-1}
\cline{4-5}
Llava:7b &  \multirow{6}{1cm}{\centering Agent}&  & Critical & An issue with the network interface is causing output drops on the device, which could impact network performance and reliability.
\\
\cline{1-1}
\cline{4-5}
IBM's Granite3.2:8b & & & Non-critical & The device's data rates, packet rates, and reliability metric all indicate normal operation with no unexpected spikes or drops in traffic. Although the load is high at 26.0, it does not signify a fault but heavy utilization, supported by the absence of critical-indicating peak values (peak-input/output-data-rate, peak-input/output-packet-rate).
\\
\cline{1-1}
\cline{4-5}
Microsoft's Phi4:14b &  & & Non-critical & The interface operates within expected parameters, with no indicators of bandwidth saturation, significant packet loss, or unreliability. Metrics such as data rate and load are normal, and reliability is at maximum without errors affecting performance. Despite some high values mentioned elsewhere, none relate to a critical fault for this specific instance.
\\ 
\hline \hline
\end{tabular}
\caption{Fault severity and cause analysis for agents enabled with large language models. The agents classify the query as {\em Non-critical} or {\em Critical} and briefly summarize the fault cause. The agents have access to network telemetry logs that were vector-embedded and stored in ChromaDB. Note that the large language models generated the text in the {\em Cause Summary}.}
\label{tab:fault_cause-agent}
\end{center}
\end{table*}

\begin{table*}
\tiny
\begin{center}
\renewcommand{\arraystretch}{1}
\setlength{\tabcolsep}{2pt}
\begin{tabular}{|c| c |p{6cm} |c|p{6cm}|}
\hline
\textbf{Models} & \textbf{Actor} & \textbf{Query} & \textbf{Severity} & \textbf{Cause Summary} \\ 
\hline
\hline
Llama3.2 & & {``Predict whether Device with 'name': 'Cisco-IOS-XR-infra-statsd-oper:infra-statistics/interfaces/interface/latest/generic-counters', 'time': '2017-09-01 10:35:19.556000', 'EncodingPath': 'Cisco-IOS-XR-infra-statsd-oper:infra-statistics/interfaces/interface/latest/generic-counters', 'Producer': 'leaf3', 'bytes-received': '934494113183718.0', 'bytes-sent': '2736075649406230.0', 'carrier-transitions': '0.0', 'crc-errors': '0.0', 'input-drops': '0.0', 'input-errors': '0.0', 'input-ignored-packets': '0.0', 'input-queue-drops': '0.0', 'interface-name': 'HundredGigE0/0/0/5', 'output-buffer-failures': '0.0', 'output-drops': '19401038.0', 'output-errors': '0.0', 'output-queue-drops': '0.0', 'packets-received': '224932326520.0', 'packets-sent': '1573333104341.0' has a Critical fault or Non-critical fault.''} & Critical & High data rate on the interface combined with high packet loss indicates a potential interface issue.\\ 
\cline{1-1}
\cline{4-5}
DeepSeek-R1 & &   & Critical & The high number of output drops (over 19 million) suggests that there may be an issue with data transmission capacity or buffer overflow on the output link interface, potentially indicating hardware exhaustion or configuration errors leading to packet loss.
\\ \cline{1-1}
\cline{4-5}
Llava:7b & \multirow{6}{1cm}{\centering Agent}  & & Critical & The interface has experienced output queue drops, which can result in lost data and reduced network performance. This can have serious implications for the functionality of the network device.
\\
\cline{1-1}
\cline{3-5}
Mistral &  & ``Predict whether Device with 'name': 'Cisco-IOS-XR-infra-statsd-oper:infra-statistics/interfaces/interface/latest/generic-counters', 'time': '2017-09-01 10:17:09.336000', 'EncodingPath': 'Cisco-IOS-XR-infra-statsd-oper:infra-statistics/interfaces/interface/latest/generic-counters', 'Producer': 'leaf3', 'bytes-received': '5344118565224976.0', 'bytes-sent': '5096276026294021.0', 'carrier-transitions': '0.0', 'crc-errors': '390.0', 'input-drops': '0.0', 'input-errors': '390.0', 'input-ignored-packets': '0.0', 'input-queue-drops': '0.0', 'interface-name': 'HundredGigE0/0/0/30', 'output-buffer-failures': '0.0', 'output-drops': '2573752423.0', 'output-errors': '0.0', 'output-queue-drops': '0.0', 'packets-received': '1974518315803.0', 'packets-sent': '2244342803311.0' has a Critical fault or Non-critical fault.'' & Critical & Potential network congestion, cable problems or misconfiguration due to high input errors and output drops on the 'HundredGigE0/0/0/30' interface. Additionally, investigate possible causes for CRC errors across all interfaces. \\ 
\hline \hline
\end{tabular}
\caption{Fault severity and cause analysis: The table presents the text queries classified as `Critical' by different LLMs. The LLMs generated the text in the {\em Cause Summary}.}
\label{tab1:Critical-fault_cause-agent}
\end{center}
\end{table*}

\begin{table*}
\tiny
\begin{center}
\renewcommand{\arraystretch}{1}
\setlength{\tabcolsep}{2pt}
\begin{tabular}{|c| c |p{6cm} |c|p{6cm}|}
\hline
\textbf{Models} & \textbf{Actor} & \textbf{Query} & \textbf{Severity} & \textbf{Cause Summary} \\ 
\hline
\hline
IBM's Granite3.2:8b & & ``Predict whether Device with 'name': 'Cisco-IOS-XR-infra-statsd-oper:infra-statistics/interfaces/interface/latest/data-rate', 'time': '2017-09-01 10:31:39.580000', 'EncodingPath': 'Cisco-IOS-XR-infra-statsd-oper:infra-statistics/interfaces/interface/latest/data-rate', 'Producer': 'spine2', 'bandwidth': '100000000.0', 'input-data-rate': '10273100.0', 'input-load': '26.0', 'input-packet-rate': '227678.0', 'interface-name': 'HundredGigE0/0/0/17', 'load-interval': '0.0', 'output-data-rate': '12093337.0', 'output-load': '30.0', 'output-packet-rate': '408305.0', 'peak-input-data-rate': '0.0', 'peak-input-packet-rate': '0.0', 'peak-output-data-rate': '0.0', 'peak-output-packet-rate': '0.0', 'reliability': '255.0' has a Critical fault or Non-critical fault.'' & Critical & The high reliability metric (255) indicates a potential interface failure, which can lead to significant network disruptions. Despite no individual metrics indicating critical faults, the overall non-responsive or malfunctioning state of 'HundredGigE0/0/17' is classified as critical due to its impact on network operations.
\\
\cline{1-1}
\cline{3-5}
Microsoft's Phi4:14b &  & Predict whether Device with 'name': 'Cisco-IOS-XR-infra-statsd-oper:infra-statistics/interfaces/interface/latest/generic-counters', 'time': '2017-09-01 10:18:59.451000', 'EncodingPath': 'Cisco-IOS-XR-infra-statsd-oper:infra-statistics/interfaces/interface/latest/generic-counters', 'Producer': 'leaf3', 'bytes-received': '5344118565224976.0', 'bytes-sent': '5096276026294021.0', 'carrier-transitions': '0.0', 'crc-errors': '390.0', 'input-drops': '0.0', 'input-errors': '390.0', 'input-ignored-packets': '0.0', 'input-queue-drops': '0.0', 'interface-name': 'HundredGigE0/0/0/30', 'output-buffer-failures': '0.0', 'output-drops': '2573752423.0', 'output-errors': '0.0', 'output-queue-drops': '0.0', 'packets-received': '1974518315803.0', 'packets-sent': '2244342803311.0' has a Critical fault or Non-critical fault & Critical & There is a possibility that the metrics are being incorrectly scaled or recorded due to misconfiguration or reporting inaccuracies by the telemetry system.
\\ 
\hline \hline
\end{tabular}
\caption{Fault severity and cause analysis: The table presents the text queries classified as `Critical' by different LLMs. The LLMs generated the text in the {\em Cause Summary}.}
\label{tab2:Critical-fault_cause-agent}
\end{center}
\end{table*}

\begin{table*}
\tiny
\begin{center}
\renewcommand{\arraystretch}{1}
\setlength{\tabcolsep}{4pt}
\begin{tabular}{|c| c |p{5.6cm} |c|p{5.6cm}|}
\hline
\textbf{Models} & \textbf{Actor} & \textbf{Query} & \textbf{Severity} & \textbf{Cause Summary} \\ 
\hline
\hline
Llama3.2 &  & ``Predict whether Device with 'name': 'Cisco-IOS-XR-infra-statsd-oper:infra-statistics/interfaces/interface/latest/generic-counters', 'time': '2017-09-01 10:35:46.985999', 'EncodingPath': 'Cisco-IOS-XR-infra-statsd-oper:infra-statistics/interfaces/interface/latest/generic-counters', 'Producer': 'leaf3', 'bytes-received': '659311227161861.0', 'bytes-sent': '388892590647867.0', 'carrier-transitions': '0.0', 'crc-errors': '0.0', 'input-drops': '0.0', 'input-errors': '0.0', 'input-ignored-packets': '0.0', 'input-queue-drops': '0.0', 'interface-name': 'HundredGigE0/0/0/8', 'output-buffer-failures': '0.0', 'output-drops': '0.0', 'output-errors': '0.0', 'output-queue-drops': '0.0', 'packets-received': '165436307368.0', 'packets-sent': '112420913236.0' has a Critical fault or Non-critical fault.'' & Critical & No Feedback \\ 
\cline{1-1}
\cline{3-5}
DeepSeek-R1 &  & \multirow{6}{5.6cm}{Predict whether Device with 'name': 'Cisco-IOS-XR-infra-statsd-oper:infra-statistics/interfaces/interface/latest/generic-counters', 'time': '2017-09-01 10:17:40.959000', 'EncodingPath': 'Cisco-IOS-XR-infra-statsd-oper:infra-statistics/interfaces/interface/latest/generic-counters', 'Producer': 'spine2', 'bytes-received': '765732261758138.0', 'bytes-sent': '692455587027158.0', 'carrier-transitions': '0.0', 'crc-errors': '0.0', 'input-drops': '0.0', 'input-errors': '0.0', 'input-ignored-packets': '0.0', 'input-queue-drops': '0.0', 'interface-name': 'HundredGigE0/0/0/14', 'output-buffer-failures': '0.0', 'output-drops': '0.0', 'output-errors': '0.0', 'output-queue-drops': '0.0', 'packets-received': '257183573922.0', 'packets-sent': '295378443912.0' has a Critical fault or Non-critical fault.} & Non-critical & No significant issues observed with the interface operational state or data rates within acceptable limits. \\ \cline{1-1}
\cline{4-5}
Mistral & \multirow{6}{1cm}{\centering Critic} &  & Non-critical & The given query data does not indicate a critical fault in the device's interface. The high traffic and data rates shown can be seen as heavy load, which is generally considered non-critical, but monitoring for potential issues in the near future is recommended. To mitigate this, network administrators might consider optimizing traffic or upgrading bandwidth accordingly. \\ \cline{1-1}
\cline{4-5}
Llava:7b &  &  & Non-critical & No Feedback
\\
\cline{1-1}
\cline{4-5}
IBM's Granite3.2:8b & & & Non-critical & The provided data does not indicate any critical faults on interface HundredGigE0/0/0/14. All counters such as bytes-received, bytes-sent, carrier-transitions, crc-errors, input-drops, input-errors, input-ignored-packets, input-queue-drops, output-buffer-failures, output-drops, output-errors, and output-queue-drops are either zero or within normal operation ranges. The interface is up and functioning normally according to the given metrics.
\\
\cline{1-1}
\cline{4-5}
Microsoft's Phi4:14b &  & & Non-critical & No Feedback
\\ 
\hline \hline
\end{tabular}
\caption{Fault severity and cause analysis for critics enabled with large language models. The critics evaluate the responses of agents and send them feedback for improvement. They also classify the query as {\em Non-critical} or {\em Critical} and briefly summarize the fault cause. The critics also have access to network telemetry logs that were vector-embedded and stored in ChromaDB. Note that the large language models generated the text in the {\em Cause Summary}.}
\label{tab:fault_cause-critic}
\end{center}
\end{table*}

The evolution of the ratio of total cost obtained by the developed algorithm and the total optimal cost obtained by the CVX solver $\frac{\sum_{i=1}^N f_i(x_{1i}(k), x_{2i}(k)) + \sum_{j=1}^2 \sum_{\nu =1}^{C_j} g_{j \nu} (y_{j \nu} (k))}{\sum_{i=1}^N f_i(x_{1i}^*, x_{2i}^*) + \sum_{j=1}^2 \sum_{\nu =1}^{C_j} g_{j \nu} (y_{j \nu}^*)}$ is shown in Figure \ref{fig:total_cost}. The ratio converges close to $1$. 

Moreover, Figure \ref{fig:feedback_signal_agents} shows the convergence of the feedback signals $\Theta_1(k)$ and $\Theta_2(k)$ for AI agents completing tasks of modalities $1$ and $2$, respectively. The central server broadcasts the feedback signals. Analogously, the central server broadcasts the feedback signals for AI critics evaluating tasks of modalities $1$ and $2$, denoted by $\Theta^c_1(k)$ and $\Theta^c_2(k)$. The convergence result is shown in Figure \ref{fig:feedback_signal_critics}.

We now present the results for network telemetry fault detection and severity and cause analysis. We chose the network telemetry dataset $4$ of \cite{Putina2021} collected on network devices. The pre-processing steps are described in detail in the previous subsection. As discussed previously, we present fault detection using two approaches: agents enabled with Classical machine learning techniques and agents and critics enabled with generative AI's foundation models. For the generative AI's foundation models we use large language models (LLMs), such as Llama3.2, DeepSeek-R1, Granite, etc. For classical machine learning, we use Extreme Gradient Boosting (XG Boosting). We evaluate the generative AI's foundation models or LLMs-based methods using two techniques: $(a)$ with retrieval augmented generation (RAG) and $(b)$ without RAG. RAG is much slower, as it needs to search and retrieve results similar to user queries in the database and send these to the LLM as context examples.
We use the retrieval augmented generation (RAG) technique to provide context to the large language models. In RAG, the queries are converted into vector embeddings that represent the queries numerically. These embedded values are then stored in a database; we use ChromaDB. Furthermore, when a new query is passed for analysis, it is matched with the queries in the database; the closest $3$ matches are retrieved and passed to the LLM as examples with corresponding fault labels. The LLM then considers this example and processes the input query from the CSV file for fault analysis. We use {\em mxbai-embed-large} embedding technique to convert text queries into vector embeddings. For the log files, we create another instance of the ChromaDB. In the original dataset \cite{Putina2021}, the log files are saved with producer names. Our algorithm extracts the names of the producers from the file names, converts the log files with names into vector embeddings, and stores them in the ChromaDB database. 

At the fault detection stage, the fault-detecting AI agents and critics can access ChromaDB with text queries and their fault labels. 

The following subsections present a few outputs and analyses on fault detection.

\subsubsection{Fault Detection}

The following is one of the AI agent's output at the fault detection stage, the agents and critics are enabled with large language models. The examples include  Query ID, Query, and the agent's Response.
\begin{example}[Query with fault detecting agent's response]
{\small    19,``Device with 'name': `Cisco-IOS-XR-infra-statsd-oper:infra-statistics/interfaces/interface/latest/data-rate', `time': `2017-09-01 10:18:50.139000', `EncodingPath': `Cisco-IOS-XR-infra-statsd-oper:infra-statistics/interfaces/interface/latest/data-rate', `Producer': `leaf8', `bandwidth': `100000000.0', `input-data-rate': `6604818.0', `input-load': `16.0', `input-packet-rate': `235248.0', `interface-name': `HundredGigE0/0/0/13', `load-interval': `0.0', `output-data-rate': `12044697.0', `output-load': `30.0', `output-packet-rate': `467318.0', `peak-input-data-rate': `0.0', `peak-input-packet-rate': `0.0', `peak-output-data-rate': `0.0', `peak-output-packet-rate': `0.0', `reliability': `255.0''',Fault}
\end{example}

The following is one of the AI critic's output at the fault detection stage. The critic is enabled with a large language model (Llama3.2). The examples include  Query ID, Query, agent's Response, and  Critic's evaluation.
\begin{example}
 {\small   33,``Device with `name': 
 
 `Cisco-IOS-XR-infra-statsd-oper:infra-statistics/interfaces/interface/latest/data-rate', 'time': `2017-09-01 10:23:07.295000', `EncodingPath':
 
  `Cisco-IOS-XR-infra-statsd-oper:infra-statistics/interfaces/interface/latest/data-rate', `Producer': `leaf3', `bandwidth': `100000000.0', `input-data-rate': `0.0', `input-load': `0.0', `input-packet-rate': `0.0', `interface-name': `HundredGigE0/0/0/31', `load-interval': `0.0', `output-data-rate': `0.0', `output-load': `0.0', `output-packet-rate': `0.0', `peak-input-data-rate': `0.0', `peak-input-packet-rate': `0.0', `peak-output-data-rate': `0.0', `peak-output-packet-rate': `0.0', `reliability': `255.0''',No Fault,Correct}
\end{example}

We evaluated the fault detection step with AI agents enabled with XG Boosting and AI agents and critics enabled with large language models with RAG and without RAG. Table \ref{tab:fault_detection} presents accuracy, F1 score, precision, and recall for fault detection by both approaches: AI agents enabled with XG Boosting and AI agents enabled with large language models. We observe that agents enabled with XG Boosting are faster and more accurate in predicting faults than those enabled with large language models. 

\subsubsection{Fault severity and Cause analysis}
For the fault severity and cause analysis, we executed the algorithms with six large language models: Meta's Llama3.2, DeepSeek-R1, Mistral, IBM's Granite3.2:8b, Microsoft's Phi4:14b, and Llava:7b. The chosen large language models are open-access and are supported by Ollama. 

Table \ref{tab:fault_cause-agent} presents the fault severity classification and cause summary by AI agents for the randomly chosen text queries from CSV files. Analogously, the fault severity classification and cause summary by AI critics is presented in Table \ref{tab:fault_cause-critic}. Additionally, Tables \ref{tab1:Critical-fault_cause-agent} and \ref{tab2:Critical-fault_cause-agent} present faults classified as `Critical' by AI agents enabled by different LLMs. Note that the LLMs generated the text in the {\em Cause Summary} column of the tables. 

Further analysis on proportion of faults classified as critical or non-critical by different large language models is presented in Figure \ref{fig:piechart_Critical}, visualized through pie-charts. We observe that Llama3.2 classifies the highest percentage of faults as critical, whereas Granite 3.2:8b classifies the least.
\begin{figure*}
    \centering
    \includegraphics[width=0.9\linewidth]{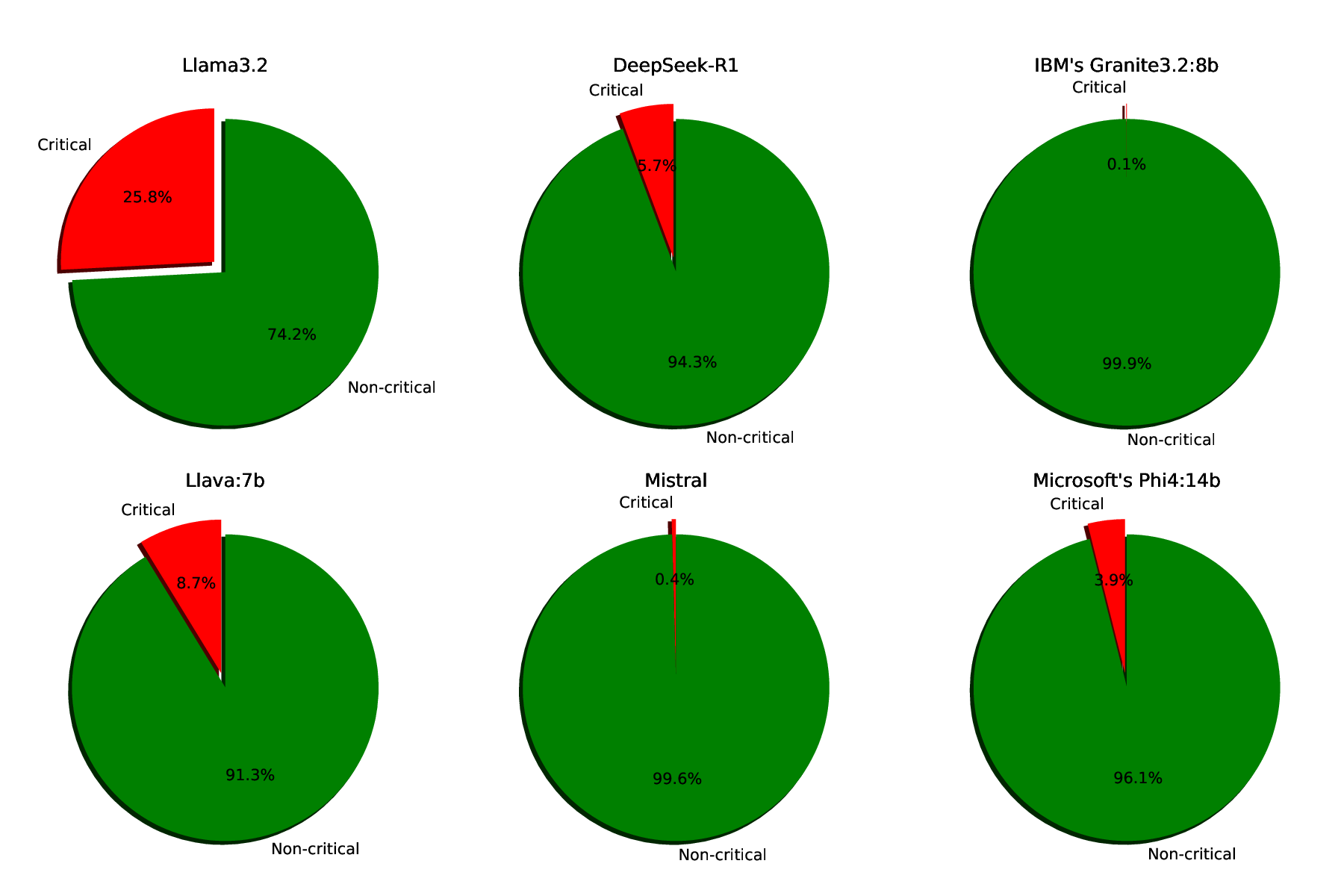}
    \caption{The proportion of critical and non-critical faults classified by various open-source large language models. The green slice represents the proportion of non-critical faults, and the exploded red slice represents the proportion of critical faults.}
    \label{fig:piechart_Critical}
\end{figure*}
\begin{figure*}
\includegraphics[width=1.0\linewidth]{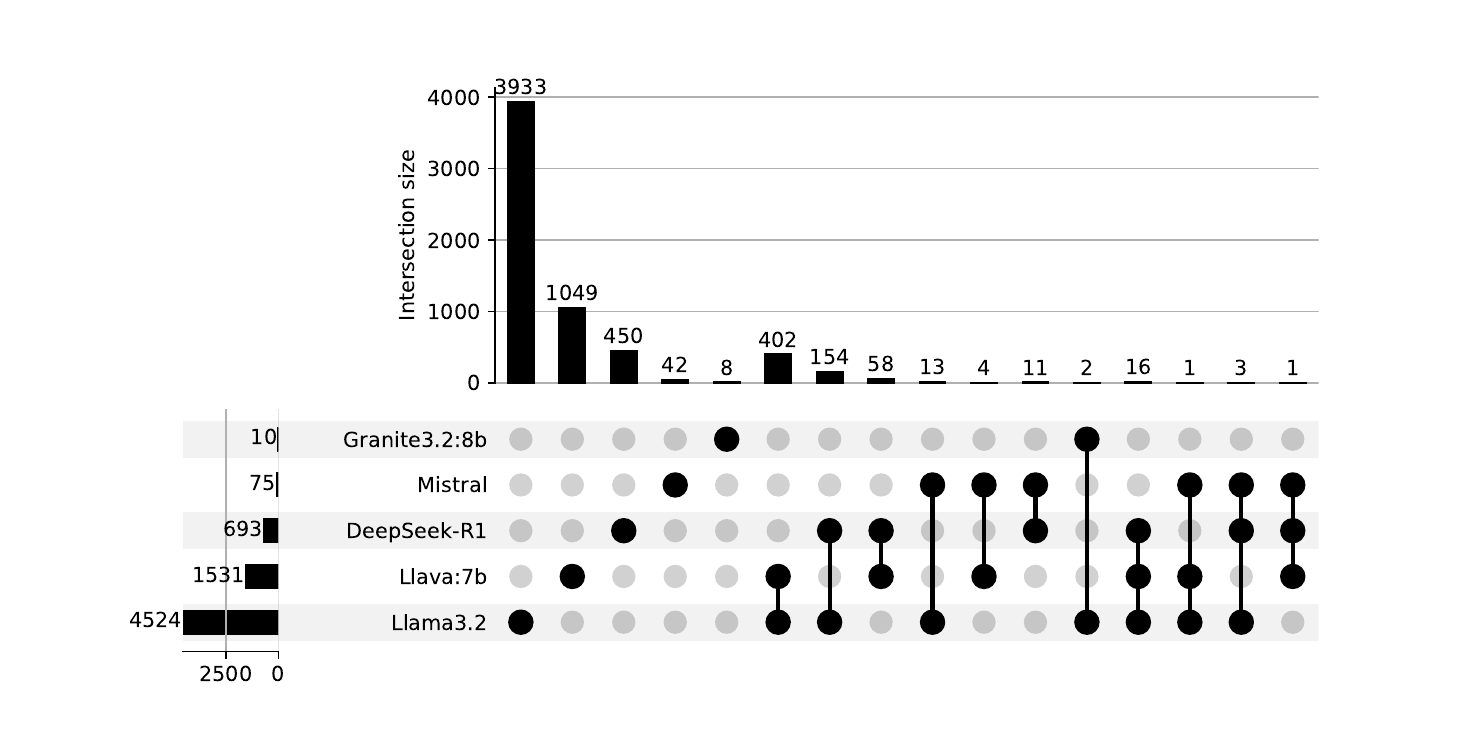}
    \caption{The UpSet plot illustrates the intersection of sets of faults classified as critical by various LLMs. The bars on the bottom left represent the number of critical faults classified by the LLMs. The bars above the single solid dot represent the number of faults classified as critical only by that large language model (a disjoint set). The connecting solid dots represent the common queries classified as critical by the connecting LLMs.}
    \label{Fig:Upset_plot}
\end{figure*}
UpSet plot in Figure \ref{Fig:Upset_plot} illustrates the distribution of faults classified as `Critical' by different large language models. Llama3.2 classifies $4524$ faults as critical, whereas Granite3.2:8b classifies $10$ as critical. Furthermore, $402$ faults are classified as critical by both Llama3.2 and Llava:7b. Similarly, Llama3.2 and DeepSeek-R1 classify $154$ faults as critical. Furthermore, Llama3.2, Llava:7b, and DeepSeek-R1 classify $16$ common fault queries as critical. Further observations can be made for other common queries.

\section{Conclusion} \label{conc}

We developed stochastic iterative algorithms to control the number of active AI agents and critics in a multi-agent system to complete multimodal tasks. The AI agents and critics are either classical machine learning or generative AI's foundation model enabled. The AI agent's cost functions are coupled through the active states of different modalities. The critic evaluates the response of the AI agents and sends feedback on the evaluation to the agent to improve its performance. The critics' cost function depends on the task modality it evaluates. AI agents and critics collaborate with a central server that keeps track of the number of active AI agents and critics to perform a modality task. Additionally, the central server sends feedback signals to the network, which helps AI agents and critics calculate their probabilistic intent to be active in performing a modality task at a time step or not. Eventually, AI agents and critics obtain optimal values over long-term average active states.  

We used the developed algorithms for fault detection, severity, and cause analysis on a real-world network telemetry dataset. For fault detection, the agents enabled with the classical machine learning approach have better accuracy, F1 score, precision, and recall than the large language model-enabled agents. Agents and critics were enabled by large language models for fault severity and cause analysis. The algorithms were evaluated with multiple large language models---Llama3.2, DeepSeek-R1, Microsoft's Phi4, IBM's granite3.2:8b, Mistral, and Llava:7b. It would be exciting to implement the algorithms on real network devices and detect the faults in real-time and analyze the causes. Moreover, the algorithm can be implemented in several application domains, such as smart energy systems, fault detection, and cause analysis in industrial applications.

\section{Acknowledgement}
This work is partially supported by SheQAI Research and the University of Alberta-Huawei Joint Innovation Centre (project number ZAI7F).


\begin{thebibliography}{10}

\bibitem{Syed2018}
Syed Eqbal Alam, Robert Shorten, Fabian Wirth, and Jia Yuan Yu.
\newblock Communication-efficient distributed multi-resource allocation.
\newblock In {\em {IEEE} International Smart Cities Conference ({ISC2})}, pages
  1--8, Sep. 2018.

\bibitem{Syed2018_B}
Syed Eqbal Alam, Robert Shorten, Fabian Wirth, and Jia Yuan Yu.
\newblock Distributed algorithms for {Internet-of-Things} enabled prosumer
  markets: {A} control theoretic perspective.
\newblock In {\em Analytics for the Sharing
  Economy: Mathematics, Engineering and Business Perspectives}. Springer, 2020.

\bibitem{Alam2023_CDC}
Syed~Eqbal Alam and Dhirendra Shukla.
\newblock Communication-efficient allocation of multiple indivisible resources
  in a federated multi-agent system.
\newblock In {\em 62nd {IEEE} Conference on Decision and Control, {CDC} 2023,
  Singapore, December 13-15, 2023}, pages 5279--5285. 2023.

\bibitem{Alayrac2024}
Jean-Baptiste Alayrac, Jeff Donahue, Pauline Luc, Antoine Miech, Iain Barr,
  Yana Hasson, Karel Lenc, Arthur Mensch, Katie Millicah, Malcolm Reynolds,
  Roman Ring, Eliza Rutherford, Serkan Cabi, Tengda Han, Zhitao Gong, Sina
  Samangooei, Marianne Monteiro, Jacob Menick, Sebastian Borgeaud, Andrew
  Brock, Aida Nematzadeh, Sahand Sharifzadeh, Mikolaj Binkowski, Ricardo
  Barreira, Oriol Vinyals, Andrew Zisserman, and Karen Simonyan.
\newblock Flamingo: a visual language model for few-shot learning.
\newblock In {\em Proceedings of the 36th International Conference on Neural
  Information Processing Systems}, NIPS '22, 2024.

\bibitem{Arzani2018}
Behnaz Arzani, Selim Ciraci, Luiz Chamon, Yibo Zhu, Hongqiang Liu, Jitu Padhye,
  Boon~Thau Loo, and Geoff Outhred.
\newblock 007: democratically finding the cause of packet drops.
\newblock In {\em Proceedings of the 15th USENIX Conference on Networked
  Systems Design and Implementation}, pages 419--435, 2018.

\bibitem{Ran2020}
Ran Ben~Basat, Sivaramakrishnan Ramanathan, Yuliang Li, Gianni Antichi, Minian
  Yu, and Michael Mitzenmacher.
\newblock Pint: Probabilistic in-band network telemetry.
\newblock In {\em Proceedings of the Annual Conference of the ACM Special
  Interest Group on Data Communication on the Applications, Technologies,
  Architectures, and Protocols for Computer Communication}, SIGCOMM '20, pages
  662--680, 2020.

\bibitem{Rishi2021}
Rishi Bommasani, Drew~A. Hudson, Ehsan Adeli, Russ~B. Altman, Simran Arora,
  Sydney von Arx, Michael~S. Bernstein, Jeannette Bohg, Antoine Bosselut, Emma
  Brunskill, Erik Brynjolfsson, Shyamal Buch, Dallas Card, Rodrigo Castellon,
  Niladri~S. Chatterji, Annie~S. Chen, Kathleen Creel, Jared~Quincy Davis,
  Dorottya Demszky, Chris Donahue, Moussa Doumbouya, Esin Durmus, Stefano
  Ermon, John Etchemendy, Kawin Ethayarajh, Li~Fei{-}Fei, Chelsea Finn, Trevor
  Gale, Lauren~E. Gillespie, Karan Goel, Noah~D. Goodman, Shelby Grossman, Neel
  Guha, Tatsunori Hashimoto, Peter Henderson, John Hewitt, Daniel~E. Ho, Jenny
  Hong, Kyle Hsu, Jing Huang, Thomas Icard, Saahil Jain, Dan Jurafsky,
  Pratyusha Kalluri, Siddharth Karamcheti, Geoff Keeling, Fereshte Khani, Omar
  Khattab, Pang~Wei Koh, Mark~S. Krass, Ranjay Krishna, Rohith Kuditipudi, and
  et~al.
\newblock On the opportunities and risks of foundation models.
\newblock {\em CoRR}, abs/2108.07258, 2021.

\bibitem{Boyd2004}
Stephen Boyd and Lieven Vandenberghe.
\newblock {\em Convex Optimization}.
\newblock Cambridge University Press, 2004.

\bibitem{Taicheng2024}
Taicheng Guo, Xiuying Chen, Yaqi Wang, Ruidi Chang, Shichao Pei, Nitesh~V.
  Chawla, Olaf Wiest, and Xiangliang Zhang.
\newblock Large language model based multi-agents: A survey of progress and
  challenges.
\newblock In {\em Proceedings of the Thirty-Third International Joint
  Conference on Artificial Intelligence, {IJCAI-24}}, pages 8048--8057, 8 2024.

\bibitem{Hyun2019}
Jonghwan Hyun, Nguyen Van~Tu, Jae-Hyoung Yoo, and James Won-Ki Hong.
\newblock Real-time and fine-grained network monitoring using in-band network
  telemetry.
\newblock {\em International Journal of Network Management}, 29(6), 2019.

\bibitem{Landau-Feibish2025}
Shir Landau-Feibish, Zaoxing Liu, and Jennifer Rexford.
\newblock Compact data structures for network telemetry.
\newblock {\em ACM Comput. Surv.}, 57(8), March 2025.

\bibitem{Junyou2024}
Junyou Li, Qin Zhang, Yangbin Yu, QIANG FU, and Deheng Ye.
\newblock More agents is all you need.
\newblock {\em Transactions on Machine Learning Research}, 2024.

\bibitem{Yuliang2019}
Yuliang Li, Rui Miao, Hongqiang~Harry Liu, Yan Zhuang, Fei Feng, Lingbo Tang,
  Zheng Cao, Ming Zhang, Frank Kelly, Mohammad Alizadeh, and Minlan Yu.
\newblock Hpcc: high precision congestion control.
\newblock In {\em Proceedings of the ACM Special Interest Group on Data
  Communication}, SIGCOMM '19, pages 44--58, 2019.

\bibitem{Mishra2024}
Mayank Mishra, Matt Stallone, Gaoyuan Zhang, Yikang Shen, Aditya Prasad,
  Adriana~Meza Soria, Michele Merler, Parameswaran Selvam, Saptha Surendran,
  Shivdeep Singh, Manish Sethi, Xuan-Hong Dang, Pengyuan Li, Kun-Lung Wu, Syed
  Zawad, Andrew Coleman, Matthew White, Mark Lewis, Raju Pavuluri, Yan Koyfman,
  Boris Lublinsky, Maximilien de~Bayser, Ibrahim Abdelaziz, Kinjal Basu, Mayank
  Agarwal, Yi~Zhou, Chris Johnson, Aanchal Goyal, Hima Patel, Yousaf Shah,
  Petros Zerfos, Heiko Ludwig, Asim Munawar, Maxwell Crouse, Pavan Kapanipathi,
  Shweta Salaria, Bob Calio, Sophia Wen, Seetharami Seelam, Brian Belgodere,
  Carlos Fonseca, Amith Singhee, Nirmit Desai, David~D. Cox, Ruchir Puri, and
  Rameswar Panda.
\newblock Granite code models: A family of open foundation models for code
  intelligence, 2024.

\bibitem{Putina2021}
Andrian Putina and Dario Rossi.
\newblock Online anomaly detection leveraging stream-based clustering and
  real-time telemetry.
\newblock {\em IEEE Transactions on Network and Service Management},
  18(1):839--854, 2021.

\bibitem{Tan2021}
Lizhuang Tan, Wei Su, Wei Zhang, Jianhui Lv, Zhenyi Zhang, Jingying Miao,
  Xiaoxi Liu, and Na~Li.
\newblock In-band network telemetry: A survey.
\newblock {\em Computer Networks}, 186, 2021.

\bibitem{Wirth2014}
F.~Wirth, S.~St\"{u}dli, J.~Y. Yu, M.~Corless, and R.~Shorten.
\newblock Nonhomogeneous place-dependent {M}arkov chains, unsynchronised
  {AIMD}, and optimisation.
\newblock {\em J. ACM}, 66(4):24:1--24:37, 2019.

\bibitem{Minlan2019}
Minlan Yu.
\newblock Network telemetry: towards a top-down approach.
\newblock {\em SIGCOMM Comput. Commun. Rev.}, 49(1):11--17, 2019.


\bibitem{Borkar2008}
V. S. Borkar.
\newblock Stochastic Approximation.
\newblock {\em Cambridge University Press, New York, NY, USA}, 2008.


\bibitem{Hoeffding1963}
W.~Hoeffding.
\newblock Probability inequalities for sums of bounded random
variables.
\newblock {\em Journal of the American Statistical Association}, 58(301):13--30, 1963.



\end{thebibliography}

\end{document}